\documentclass{article}

\usepackage[preprint]{neurips_2026}

\usepackage[utf8]{inputenc}
\usepackage[T1]{fontenc}
\usepackage{hyperref}
\usepackage{url}
\usepackage{booktabs}
\usepackage{amsfonts}
\usepackage{nicefrac}
\usepackage{microtype}
\usepackage{xcolor}

\usepackage{graphicx}
\usepackage{subcaption}
\usepackage{amsmath}
\usepackage{amssymb}
\usepackage{mathtools}
\usepackage{amsthm}
\usepackage{tikz-cd}
\usepackage{tikz}
\usepackage{tcolorbox}
\usepackage{float}
\usepackage{wrapfig}
\usepackage{placeins}

\usepackage[capitalize,noabbrev]{cleveref}

\theoremstyle{plain}
\newtheorem{theorem}{Theorem}[section]

\newtheorem{corollary}[theorem]{Corollary}
\theoremstyle{definition}

\theoremstyle{remark}
\newtheorem{remark}[theorem]{Remark}

\newcommand{\R}{\mathbb{R}}
\newcommand{\Tr}{\mathcal{T}}
\DeclareMathOperator{\tr}{Tr}
\DeclareMathOperator{\pdet}{pdet}
\DeclareMathOperator{\im}{im}
\DeclareMathOperator{\diag}{diag}
\DeclareMathOperator{\GL}{GL}
\DeclareMathOperator{\Vol}{Vol}
\newcommand{\Onl}{\mathcal{O}_{n_f}^{L}}
\newcommand{\va}{\vec{a}}
\newcommand{\vw}{\vec{w}}

\usepackage[disable,textsize=tiny]{todonotes}

\makeatletter
\renewcommand{\@maketitle}{%
  \vbox{%
    \hsize\textwidth
    \linewidth\hsize
    \vskip 0.1in
    \@toptitlebar
    \centering
    {\LARGE\bf \@title\par}
    \@bottomtitlebar
    \begin{minipage}{0.95\textwidth}\centering
      \@author
    \end{minipage}\par
    \vskip 0.3in \@minus 0.1in
  }%
}
\makeatother

\title{The Role of Symmetry in Optimizing Overparameterized Networks}

\author{%
  \textbf{Kusha Sareen\thanks{Corresponding author: \texttt{kusha.sareen@mila.quebec}} \quad Mohammad Pedramfar \quad S\'ekou-Oumar Kaba} \\
  \textbf{Mehran Shakerinava \quad Siamak Ravanbakhsh} \\[6pt]
  McGill University \\
  Mila - Quebec Artificial Intelligence Institute
}

\begin{document}

\maketitle

\begin{abstract}
Overparameterization is central to the success of deep learning, yet the mechanisms by which it improves optimization remain incompletely understood. We analyze weight-space symmetries in neural networks and show that overparameterization introduces additional symmetries that benefit optimization in two distinct ways. First, we prove that these symmetries act as a form of diagonal preconditioning on the Hessian, enabling the existence of better-conditioned minima within each equivalence class of functionally identical solutions. Second, we show that overparameterization increases the probability mass of global minima near typical initializations, making these favourable solutions more reachable. These results offer a potential link between loss landscape geometry and simplicity bias.
Empirically, we observe wider networks have lower top eigenvalues, smaller condition numbers and faster convergence, matching our analysis.
% Experiments support our theoretical predictions: as width increases, the Hessian trace decreases, condition numbers improve, and convergence accelerates.
Our analysis provides a unified framework for understanding overparameterization and width growth as a geometric transformation of the loss landscape.
\end{abstract}
\section{Introduction}
Overparameterized neural networks consistently outperform their smaller counterparts in optimization speed and final performance, even when the target function could be represented by much smaller architectures \citep{arora2018optimization, yang2022tensor}. A major line of theory explains this via large-width limits and linearization
% : sufficiently overparameterized networks exhibit kernel-like dynamics and provable convergence of first-order methods
\citep{jacot2018ntk,allenzhu2019convergence,du2019gradient,zou2019sgd,chizat2019lazy}. Empirical studies of curvature complement this picture, showing that overparameterized networks have highly structured Hessian spectra with large near-zero subspaces and few data-dependent outliers \citep{sagun2017hessian,ghorbani2019investigation}.

The loss landscape is also shaped by \emph{symmetry}: neural networks possess large groups of function-preserving reparameterizations (permutations, rescalings, and more structured invariances) that induce equivalence classes of weights and manifolds of functionally identical minima \citep{ simsek2021geometry,kunin2021neural,grigsby2023hidden,brea2019permutation, 10.5555/3692070.3693489}. These symmetries complicate naive geometric notions such as ``sharpness'' \citep{dinh2017sharp} and help explain why seemingly distinct solutions can be connected by low-loss paths \citep{garipov2018loss,draxler2018no}.

This paper sits at the intersection of these two perspectives. We provide a more comprehensive characterization of the weight-space symmetries due to overparameterization, building on \citet{FUKUMIZU2000317, simsek2021geometry}, and show they aid optimization through two mechanisms:

\textbf{Mechanism 1: Diagonal Preconditioning.} The eigenvalues of the Hessian $H$ after a symmetry transformation are the same as those for $K^{1/2}HK^{1/2}$, where $K$ is a diagonal matrix that depends on the symmetry transformations. The right choice of $K$ can improve the Hessian conditioning, especially if eigenvectors are approximately axis-aligned.
% , a property enhanced by rotation symmetries in architectures like Transformers \citep{zhang2025permutation}.

\textbf{Mechanism 2: Volume Growth.} Overparameterization increases the total \emph{volume} of global minima near typical initializations; specifically, the probability of initializing within $\epsilon$ of a minimum grows with width once the network is sufficiently wide. Adding parameters introduces new symmetries that expand the equivalence class, making well-conditioned, reachable minima more abundant.

% Our analysis is primarily an \emph{existence result}. We characterize which minima can exist and their properties, rather than which minima gradient descent reaches. This distinction reflects an ``anthropic principle'' for optimization: we need not worry about creating badly-conditioned minima with small volume, since optimization algorithms naturally avoid them. What matters is that \emph{some} well-conditioned, high-volume minima exist, and we show that overparameterization can provide this.

\paragraph{Contributions.}
\begin{itemize}
    \item We characterize overparameterization symmetries via a matrix groupoid building on \citet{simsek2021geometry}, and identify a new pre-activation splitting symmetry for ReLU networks (\cref{sec:characterizing,sec:groupoid}).
    \item We prove these symmetries act as diagonal preconditioning on the Hessian, with explicit formulas for how gradients and eigenvalues transform. Under even splits, the product of nonzero eigenvalues is preserved while the trace can decrease, giving minima with better average-case conditioning (\cref{sec:optimization,sec:trace}).
    \item We show the probability of initializing within $\epsilon$ of a global minimum grows with width once the network is sufficiently wide (\cref{sec:volume}).
    \item Experiments on teacher-student MLPs, CNNs and Transformers, and on direct data fitting, confirm the predicted improvements in conditioning and convergence (\cref{sec:experiments}).
\end{itemize}
\section{Preliminaries}
\label{sec:prelim}
\paragraph{Parameter Space and Loss.}
Fix a depth-$L$ feedforward architecture with layer widths $n_0=d_{\mathrm{in}}, n = n_1, \dots, n_{L-1}, n_L=d_{\mathrm{out}}$.
Let $\Theta$ denote the corresponding parameter space (weights and biases). For $\theta\in\Theta$, write the realized function as
$f_\theta:\mathbb{R}^{d_{\mathrm{in}}}\to\mathbb{R}^{d_{\mathrm{out}}}$.
Given a ground-truth function $f^*$ and an input distribution $\mathcal{D}$, we consider the population risk
$\mathcal{L}(\theta)\;=\;\mathbb{E}_{x\sim\mathcal{D}}\bigl[\ell\bigl(f_\theta(x),f^*(x)\bigr)\bigr]$.
Our focus is on \emph{geometric} properties of $\mathcal{L}$ in parameter space near minima: equivalence classes of functionally identical parameters, their induced flat directions, and the curvature in directions transverse to these classes.

\paragraph{Functional Equivalence and Symmetry Orbits.}
We distinguish \emph{functional equivalence} from \emph{equivalence under a chosen symmetry family}.
Define $\theta\sim_f \theta'$ if $f_\theta=f_{\theta'}$ (equality as functions), and define $\theta\sim_{\mathcal{T}}\theta'$ if there exists $g\in\mathcal{T}$ with $g(\theta)=\theta'$, where
\begin{equation}
\mathcal{T} = \{g : \Theta \to \Theta \mid f_\theta(x) = f_{g(\theta)}(x) \ \forall x, \theta\}.
\end{equation}
Write the corresponding equivalence classes as
\[
[\theta]_f=\{\theta':\theta'\sim_f\theta\},\qquad
[\theta]_{\mathcal{T}}=\{\theta':\theta'\sim_{\mathcal{T}}\theta\}.
\]
For $\sim_{\mathcal{T}}$ to be an equivalence relation, we need $\mathrm{Id} \in \mathcal{T}$ and $g \in \mathcal{T} \iff g^{-1} \in \mathcal{T}$; transitivity follows from composition. When $\mathcal{T}$ contains all function-preserving transformations, $[\theta]_{\mathcal{T}} = [\theta]_f$; otherwise $[\theta]_{\mathcal{T}} \subset [\theta]_f$. To ensure closure, we define $\mathcal{T}$ through families of generators.

$[\theta]_{\mathcal{T}}$ is seen as the \emph{symmetry orbit} through $\theta$. At a minimum, directions tangent to this orbit typically correspond to zero-curvature directions of the Hessian, while curvature relevant for optimization lives in directions orthogonal to the orbit \citep{kunin2021neural, simsek2021geometry}.

\paragraph{Realizability.} We operate in a realizable setting, assuming the target function $f^*$ can be exactly represented by a neural network of some minimal (functional) width $n_f$. When an overparameterized network of width $n \ge n_f$ is trained to match $f^*$, global minima exist by construction. The central question is how this excess capacity reshapes the symmetry orbit structure and, through it, the curvature and reachability properties that govern optimization.

To make this setup mathematically explicit, we often adopt the terminology of a \emph{teacher-student} framework. A \emph{teacher} parameter vector $\theta^*\in\Theta_{n_f}$ defines the minimal target function $f^* = f_{\theta^*}$, while a \emph{student} network of width $n$ is trained to match it. Importantly, this is not a restrictive assumption: universal approximation theorems imply that \emph{any} target function can be represented by a network of sufficient minimal width.
\section{Characterizing Overparameterization Symmetries}
\label{sec:characterizing}

\begin{figure}[t]
    \centering
    \includegraphics[width=0.6\textwidth]{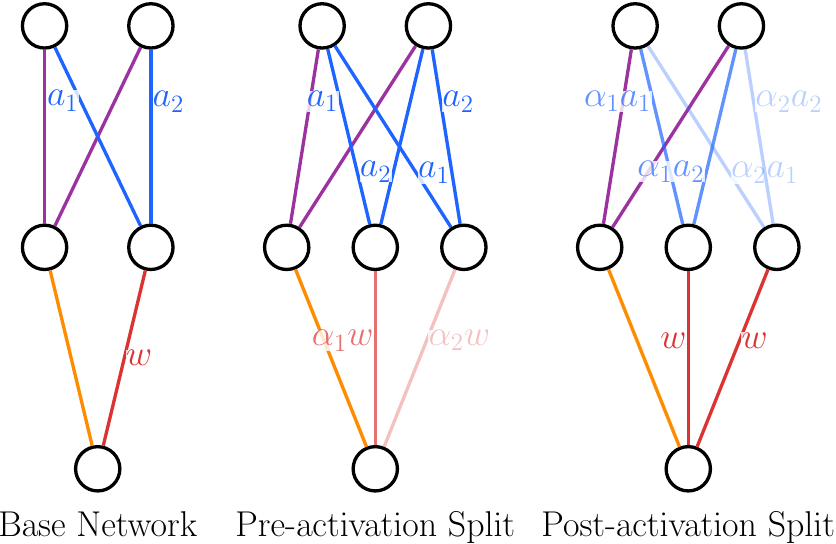}
    \caption{A neural network before and after splitting operations.}
    \label{fig:splitting}
\end{figure}

We begin by noting the definition of overparameterization used in this work, following \cite{simsek2021geometry}. We say a neural network is overparameterized if it has greater width than required to represent a given function. This is a natural notion that represents having excess capacity when studying optimization of a given function rather than over a dataset. It is related but distinct from the notion broadly relating the number of parameters to data samples.

Prior work gives a description of the overparameterized solution set, primarily for
\emph{single-hidden-layer} networks and analyzes how permutation symmetries produce manifolds of minima and symmetry-induced critical points \citep{simsek2021geometry}.
Here, we extend these results with a useful matrix viewpoint and explain what changes with depth and choice of activation.

\subsection{Fixed-width Symmetries in Nonlinear Networks}
We first remark that a number of well-studied fixed-width parameter symmetries can be seen as matrices. For a network with activation $\sigma$ to have a symmetry via matrix $M$, we require
\begin{equation} \label{eq:sym}
    M^{-1} \sigma(Mx) = \sigma(x) \,\forall x.
\end{equation}
For example, permuting hidden units in a width-$n$ layer is obtained by taking $M$ to be a permutation matrix. Similarly, for ReLU networks, neuron-wise rescaling is obtained by taking $M=\mathrm{diag}(\lambda_1,\dots,\lambda_n)$ with $\lambda_i>0$. For permutations, \cref{eq:sym} holds since $\sigma$ is applied pointwise and for ReLU, the activation is equivariant to positive rescaling.
\subsection{Overparameterization in Deep Linear Networks}
\label{sec:linear}
Deep linear networks are parameterized as $f(x) = W_L \cdots W_2 W_1 x$. It is well documented that deep linear networks have a matrix factorization symmetry that impacts optimization \citep{saxe2014exact, ghosh2025learning}. In particular, if we have fixed width $n$, between any pair $W_{k+1} W_k$ we can insert an invertible matrix $M \in \GL_n(\R)$ as $M^{-1}M$. This is the same symmetry present between $W_K$ and $W_Q$ in self-attention.

To generalize this to overparameterization, we can think of $M$ as an expansion and following contraction operation. Here, we have $M \in \R^{r \times n}$ for $r > n$ and a corresponding pseudoinverse $M^+ \in \R^{n \times r}$ such that $M^+ M = I_n$.
% This rank constraint corresponds to the intersection subspace view that connects all equivalent minima in \citet{simsek2021geometry}.
Due to the use of pseudo-inverse, rather than inverse, we have additional degrees of freedom in the choice of $M^{+}$.
The transformation $W_{k} \to MW_{k}$ and $W_{k+1} \to  W_{k+1}M^+$ preserves the function while changing the parameterization.
\subsection{General Activations: Post-Activation Splitting}
\label{sec:general}
In the case of an arbitrary activation $\sigma$, the matrices $M$ become much more restricted since one cannot arbitrarily mix hidden units while preserving $f_\theta(x)$ for all $x$.
% For some completely arbitrary activation $\sigma$, we have pre-activation $\vw x$ and final output $\va \sigma(\vw x)$.
Since we make \textit{no assumptions} about $\sigma$, the only function-preserving operations are:
\begin{enumerate}
    \item \textit{Copy} the input and split the outgoing weights: duplicating a neuron with inputs $\vw$ and outputs $\va$ into neurons $(\vw, \vw)$ with outputs $(\va', \va'')$ satisfying $a_i' + a_i'' = a_i$.
    \item Add \textit{0-type} neurons with arbitrary inputs and output weights summing to $\vec{0}$ (Appendix~\ref{app:zero_type}).
\end{enumerate}
% These correspond exactly to the symmetries identified by \citet{simsek2021geometry}.

We characterize this through a class of \textit{duplication matrices}. Here, we allow expansion by a matrix $M \in \R^{r \times n}$ for $r > n$. $M$ is a duplication matrix and repeats rows $\vec{e}_i$ of $I_n$ with multiplicity $m_i$:

\begin{equation}
M = \begin{bmatrix}
\mathbf{1}_{m_1} \otimes \vec{e}_1^\top \\
\mathbf{1}_{m_2} \otimes \vec{e}_2^\top \\
\vdots \\
\mathbf{1}_{m_n} \otimes \vec{e}_n^\top
\end{bmatrix},
\end{equation}

where $\mathbf{1}_{m_i}$ is an $m_i$-dimensional vector of ones and $\otimes$ denotes the Kronecker product. Thus, each basis vector $\vec{e}_i$ of $I_n$ is duplicated $m_i$ times, yielding a total expansion $r = \sum_i m_i$.

% The contraction is only a matrix operation in some special cases.
Here, $M^+$ can be \textit{any} pseudoinverse such that $M^+ M = I_n$.
The freedom in choosing $M^+$ corresponds to redistributing outgoing weights across the expanded layer i.e., picking $(\va', \va'')$ so $a_i' + a_i'' = a_i$ for all $i$.
Our symmetry corresponds to transforming $$W_{l} \to M W_{l} \text{ and }W_{l+1} \to W_{l+1}M^+.$$
We extend this construction with per-weight coefficient matrices $C_i$, which give a strictly richer symmetry than the per-neuron formulation in \citet{simsek2021geometry}. Details and an inverse transformation are in Appendix~\ref{app:per_weight}.

\begin{tcolorbox}[title=Example: Duplication and Pseudoinverse, colback=blue!5]
For $M = \begin{bsmallmatrix} 1 & 0 \\ 1 & 0 \\ 0 & 1 \end{bsmallmatrix}$ (duplicating the first neuron), we can have several pseudoinverses $M^+$:
\[
M^+ = \begin{bmatrix} \tfrac{1}{2} & \tfrac{1}{2} & 0 \\ 0 & 0 & 1 \end{bmatrix}, \quad
\begin{bmatrix} 0.9 & 0.1 & 0 \\ 0 & 0 & 1 \end{bmatrix}, \quad \text{etc.}
\]
The freedom in $M^+$ corresponds to how we split the output weights.
\end{tcolorbox}
\subsection{ReLU and Piecewise-Linear Activations}
\label{sec:relu}
Many practical activations are not so pathological. For instance, ReLU and its variants are linear in half the domain, which makes them equivariant to positive rescaling: $\sigma(\lambda x) = \lambda \sigma(x)$ for $\lambda > 0$. This enables an additional symmetry: \textit{pre-activation splitting}.

To split the activations evenly, we multiply the pre-activation weights by positive coefficients $[d_1, d_2, \ldots, d_m]$ such that $\sum_i d_i = 1$. If the original pre-activation is positive, the ReLU remains linear, preserving the output. If negative, the activation stays negative after splitting.

This symmetry is stricter than post-activation splitting because it requires splitting the entire incoming weight vector with the same coefficient, rather than individually for each weight.

This can be written as left-multiplying by the expansion matrix:
\begin{equation}
M' = \begin{bmatrix}
D_{m_1} \otimes \vec{e}_1^\top \\
\vdots \\
D_{m_n} \otimes \vec{e}_n^\top
\end{bmatrix},
\end{equation}
where $D_{m_i} = [d_1, \ldots, d_{m_i}]^\top$ is a column vector with $d_j > 0$ and $\sum_j d_j = 1$. The corresponding contraction transformation on $W_{l+1}$ involves copying the columns, which is represented by right-multiplying by the pseudoinverse:
\begin{equation}
M'^+ = \begin{bmatrix}
\vec{e}_1 \otimes \mathbf{1}_{m_1}^\top & \cdots & \vec{e}_n \otimes \mathbf{1}_{m_n}^\top
\end{bmatrix}.
\end{equation}

\begin{tcolorbox}[title=Example: Pre-Activation Splitting (ReLU), colback=orange!5]
Even split of input weights, duplicating outputs:
\[
M' = \begin{bmatrix} \tfrac{1}{2} & 0 \\ \tfrac{1}{2} & 0 \\ 0 & 1 \end{bmatrix}, \quad
M'^+ = \begin{bmatrix} 1 & 1 & 0 \\ 0 & 0 & 1 \end{bmatrix}
\]
Input weights are halved; output weights are copied. This preserves the function because ReLU is equivariant to positive scaling.
\end{tcolorbox}
In conjunction with post-activation splitting, this gives a family of operations generated by these matrices. We discuss which other activation functions allow pre-activation splitting in Appendix~\ref{app:activations}.
\section{The Overparameterization Groupoid}
\label{sec:groupoid}

\begin{figure}[t]
  \centering
  \begin{small}
    \begin{tikzcd}[
    row sep=1.2em,
    column sep=1.8em,
    execute at end picture={
\coordinate (a) at (-4,2.4);
\coordinate (b) at (1.5,2.4);
\coordinate (c) at (-4,-1.5);
\draw[->,blue] (a) -- (b) node [midway, above] {\tiny min.\ functional width $n_f$};
\draw[->,red] (a) -- (c) node [midway, sloped, above, rotate=180] {\tiny width $n$};
}]
    \Theta_{1,1} \arrow[d, bend right=15] \arrow[dd, bend right=30] \arrow[ddd, bend right=45] \\
    \Theta_{2,1} \arrow[u, bend right=15] \arrow[d, bend right=15] \arrow[dd, bend right=30] & \Theta_{2,2} \arrow[d, bend right=15] \arrow[dd, bend right=30] \\
    \Theta_{3,1} \arrow[u, bend right=15] \arrow[uu, bend right=30] \arrow[d, bend right=15] & \Theta_{3,2} \arrow[u, bend right=15] \arrow[d, bend right=15] & \Theta_{3,3} \arrow[d, bend right=15] \\
    \Theta_{4,1} \arrow[u, bend right=15] \arrow[uu, bend right=30] \arrow[uuu, bend right=45] & \Theta_{4,2} \arrow[u, bend right=15] \arrow[uu, bend right=30] & \Theta_{4,3} \arrow[u, bend right=15] & \Theta_{4,4} \\
    \vdots & \vdots & \vdots & \vdots & \ddots
  \end{tikzcd}
  \end{small}
  \caption{The overparameterization groupoid. $\Theta_{n, n_f}$ denotes weights of width $n$ and minimum functional width $n_f$. Arrows indicate symmetry transformations that connect equivalent parameters across the widths.}
  \label{fig:groupoid}
\end{figure}
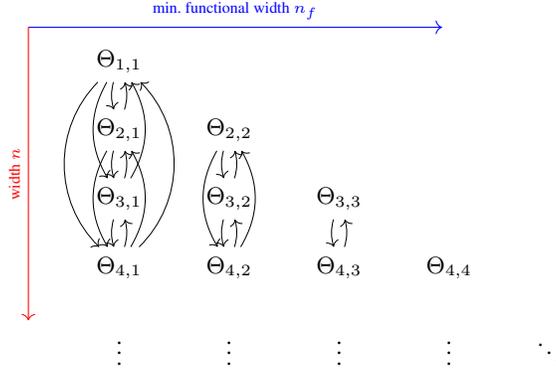

The symmetry transformations form a \emph{groupoid} rather than a group. A groupoid is a set of composable transformations following the group axioms that need not be defined on all objects. There are two reasons for using a groupoid.

\textbf{Partial composition.} A transformation $\tau_{a \to b}: \Theta_a \to \Theta_b$ can only compose with transformations originating from width $b$.

\textbf{Restricted inverses.} We cannot contract beyond the minimal width needed to represent the function.
Let $\Onl$ denote the groupoid of overparameterization operations on networks with $L$ hidden layers and minimum functional width $n_f$. The key property is that the function is preserved across the orbit: $f_\theta = f_{\theta'}$ for any $\theta' \in \theta \cdot \Onl$.

Suppose $b>a$. The groupoid consists of actions for a single layer as the set of tuples
$$\tau_{a \rightarrow b} = \{(M_{a\rightarrow b}, M^+_{a\rightarrow b})\},$$
where $M_{a\rightarrow b}$ is a valid expansion matrix and $M^+_{a\rightarrow b}$ a corresponding contraction (pseudoinverse) from \cref{sec:characterizing}. \cref{fig:groupoid} illustrates the structure: each column represents weights with a fixed minimum width $n_f$, and arrows connect equivalent parameterizations across different widths. See Appendix~\ref{app:groupoid} for a full characterization, including the per-weight coefficient form.
\section{Optimization Properties}
\label{sec:optimization}
We now analyze how overparameterization symmetries affect optimization. Note that we will rarely initialize into $\theta \cdot \Onl$ and optimize inside the orbit. Instead, the goal is to show that there exist minima with favourable properties and understand the action of width growth on the loss landscape geometry. A complete treatment can be found in Appendix~\ref{app:vectorize}.
\subsection{Gradients Under Splitting}
\label{sec:gradients}
We consider how gradients transform under splitting operations.
\begin{theorem}[Gradient Scaling]
\label{thm:gradient}
Consider a splitting operation with coefficients $\alpha_1, \alpha_2, \ldots, \alpha_m$ summing to 1, where each coefficient multiplies the full vector of incoming or outgoing weights. Let $D$ denote the set of duplicated weights. Then for the duplicated weight $w'_i$ associated with coefficient $\alpha_i$,
\[
\frac{\partial L}{\partial w'_i} = \alpha_i \frac{\partial L}{\partial w}.
\]
The gradients of all other weights, including the split weights, are unchanged.
\end{theorem}
\begin{corollary}[Even Split]
For an order-$m$ even split (all $\alpha_i = 1/m$), the gradient of each duplicated weight is $\frac{1}{m} \frac{\partial L}{\partial w}$.
\end{corollary}
% \begin{proof}[Proof of \cref{thm:gradient}]
% Using the standard notation $a^{(l)} = W^{(l)} h^{(l-1)} + b^{(l)}$ and $h^{(l)} = \sigma(a^{(l)})$ for the pre-activations and activations at layer $l$. Error backpropagates as $\delta^{(l-1)} = \sigma' \odot (W^{(l)})^\top \delta^{(l)}$. The gradient of the weights is $\frac{\partial L}{\partial W^{(l)}_{ji}} = \delta^{(l)}_j h^{(l-1)}_i$.

% For post-activation splitting, we duplicate the activations in layer $l$ and split the outgoing weights in layer $l+1$. The gradient of the outgoing split weights is unaffected since $\delta^{(l+1)}_j$ and $h^{(l)}_i$ are fixed. The gradient of the incoming duplicated weights in layer $l$ scales by $\alpha_i$ since $\delta^{(l)}_j$ is split with those coefficients.

% For pre-activation splitting, the argument is analogous with roles reversed. Finally, downstream gradients are not impacted since $M^\top \sigma' M^{+\top} = \sigma'$.
% \end{proof}
This means, for a set of transformations $\{M_l, M_l^+\}_{l=1...L}$ applied at width $n$ layers $\{W_l\}_{l=1...L}$, the gradient vector for all parameters $\theta = \text{vec}(\{W_l\}_{l=1...L}) $ transforms as $\nabla L_{\theta'} = B^T\nabla L_\theta$ for $B^T$ a rectangular block diagonal matrix given by
\begin{equation}
    B^T = \mathrm{diag}\left(I \otimes (M_1^+)^T, \dots, M_{l-1} \otimes (M_l^+)^T, \dots, M_{L-1} \otimes I \right)
\end{equation}
This comes from vectorizing the expanded weights across layers and taking the derivative.
% See Appendix~\ref{app:vectorize} for details.
\subsection{Hessians Under Splitting}
\label{sec:hessians}
\begin{figure}[t]
    \centering
    \includegraphics[width=\textwidth]{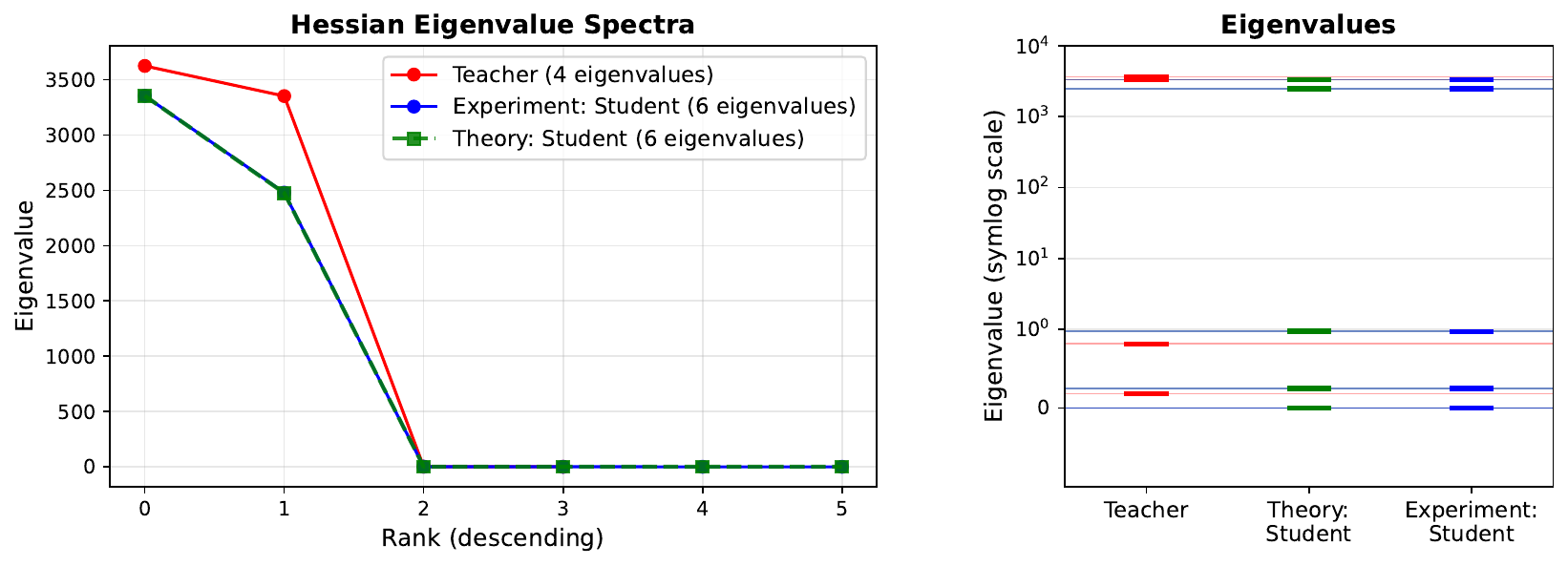}
    \caption{Empirical eigenvalue transformation for a width-2 teacher and width-3 student after optimization. \emph{Left:} ranked eigenvalue spectra. \emph{Right:} eigenvalue comparison. Eigenvalues found experimental align up to $0.2\%$ with those predicted by \cref{thm:precond}. There are two additional flat directions from splitting.}
    \label{fig:hessian_spectrum}
\end{figure}
Let $H \in \mathbb{R}^{a \times a}$ be the full-rank Hessian of the original parameterization and let $H' \in \mathbb{R}^{b \times b}$ be the expanded Hessian obtained after applying an expansion from $a\rightarrow b$ weights. Differentiating the gradient at a global minimum, we can express the expanded Hessian as $H' = B^\top H B$, where $B \in \mathbb{R}^{a \times b}$ is the matrix mapping the gradients of the expanded network back to the original parameters.

\begin{tcolorbox}[title=Example: Hessian Transformation, colback=purple!5]
For a 1-hidden-layer network with 1 hidden neuron, 1 input weight $w$ and 1 output weight $a$, duplicating $a$ and splitting $w$ with coefficients $(\alpha, 1-\alpha)$ gives (writing $\bar\alpha := 1-\alpha$)
\[
H = \begin{bmatrix} h_{ww} & h_{wa} \\ h_{wa} & h_{aa} \end{bmatrix}
\;\to\;
H' = \begin{bmatrix}
h_{ww} & h_{ww} & \alpha h_{wa} & \bar\alpha h_{wa} \\
h_{ww} & h_{ww} & \alpha h_{wa} & \bar\alpha h_{wa} \\
\alpha h_{wa} & \alpha h_{wa} & \alpha^2 h_{aa} & \alpha\bar\alpha h_{aa} \\
\bar\alpha h_{wa} & \bar\alpha h_{wa} & \alpha\bar\alpha h_{aa} & \bar\alpha^2 h_{aa}
\end{bmatrix}.
\]
The two duplicated rows/columns are everywhere scaled by $\alpha$ and $\bar\alpha$ respectively, tracing back to the gradient scaling in \cref{thm:gradient}.
% For uneven splits with coefficient $\alpha$ and $(1-\alpha)$:
% \[
% H' = \begin{bmatrix}
% h_{ww} & h_{ww} & \alpha h_{wa} & (1-\alpha) h_{wa} \\
% h_{ww} & h_{ww} & \alpha h_{wa} & (1-\alpha) h_{wa} \\
% \alpha h_{wa} & \alpha h_{wa} & \alpha^2 h_{aa} & \alpha(1-\alpha) h_{aa} \\
% (1-\alpha) h_{wa} & (1-\alpha) h_{wa} & \alpha(1-\alpha) h_{aa} & (1-\alpha)^2 h_{aa}
% \end{bmatrix}
% \]
\end{tcolorbox}
 % Appendix~\ref{app:vectorize} contains the full derivation and the general $m$-way block form.

% TODO: Theorem 5.3 (Eigenvalue Interlacing) commented out -- proof was invalid
% (B is not semi-orthogonal; Poincare separation does not apply). Need to remake
% fig:hessian_spectrum to show K-predicted eigenvalues vs. empirical, then
% replace this with a correct Ostrowski-type bound or similar.
%
% \begin{theorem}[Eigenvalue Interlacing]
% \label{thm:interlace}
% The eigenvalues of $H'$ interlace with those of $H$: for sorted eigenvalues $\lambda_i(H)$ and $\mu_i(H')$,
% \begin{equation}
% \lambda_i \geq \mu_i \geq \lambda_{n-r+i},
% \end{equation}
% where $r$ is the number of added dimensions. Additionally, $r$ zero eigenvalues are introduced (corresponding to level-sets in the loss).
% \end{theorem}
%
% \begin{proof}
% Since $H' = B^\top H B$ with $B$ semi-orthogonal, the result follows by the Poincar\'{e} separation theorem.
% \end{proof}
\subsection{Overparameterization as Diagonal Preconditioning}
\label{sec:preconditioning}
\begin{theorem}
\label{thm:precond}
    Overparameterization acts as diagonal preconditioning: the transformed Hessian $H'$ has the same nonzero eigenvalues as $HK$, equivalently $K^{1/2} H K^{1/2}$, where $K = BB^\top$ is a diagonal matrix determined by the splitting structure.
\end{theorem}

For an expansion containing a split weight and its corresponding duplicated weights, $K = BB^\top$ is diagonal. The diagonal entries correspond to split weights (scaled by $m$), duplicated weights with proportions $\alpha_1, \dots, \alpha_m$ (scaled by $\sum_{i=1}^m \alpha_i^2$), and other weights (unchanged, scaled by $1$).
\subsection{Local Volume and Pseudo-Determinant}
\label{sec:volume_properties}
The $\epsilon$-sublevel set of a minimum $\theta^*$ is locally an ellipsoid with volume $\propto (\det H)^{-1/2}$ (\cref{app:ellipsoid_volume}). Since splitting introduces flat symmetry directions, $H'$ is singular, so we measure volume \emph{modulo symmetry directions} via the pseudo-determinant $\pdet(H') = \prod_{i: \lambda_i > 0} \lambda_i$. By \cref{thm:precond}, the nonzero eigenvalues of $H'$ are exactly those of $HK$, giving $$\pdet(H') = \det(HK) = \pdet(H)\det(K).$$
\begin{theorem}[Pseudo-determinant Under Splitting]
\label{thm:pdet}
Let $H'$ be obtained from a full-rank Hessian $H$ by an $m$-way split with coefficients $\alpha_1, \ldots, \alpha_m$ summing to 1, alongside the necessary duplication of subsequent layers. Then $\pdet(H') = m \cdot \sum_i \alpha_i^2 \cdot \pdet(H) \geq \pdet(H)$, with equality iff $\alpha_i = 1/m$ for all $i$.
\end{theorem}
Thus,
the volume modulo symmetry directions is preserved under even splits, while uneven splits increase $\pdet$ and shrink the effective local volume.
\section{Creating Better-Conditioned Minima}
\label{sec:trace}
\begin{figure}[t]
    \centering
    \includegraphics[width=\textwidth]{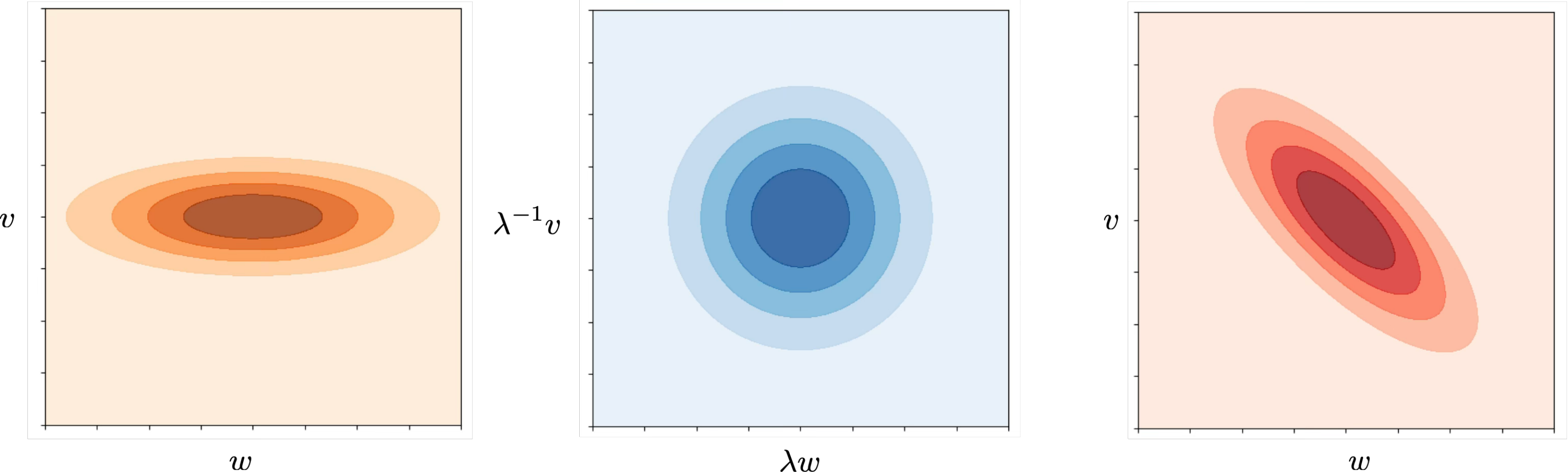}
    \caption{Diagonal preconditioning via symmetry transformations. Left: a poorly conditioned minimum $(w^*, v^*)$ with elongated level sets. Center: rescaling to $(\lambda w^*, v^*/\lambda)$ compresses the $w$-axis and stretches the $v$-axis, producing rounder level sets that are easier to optimize. Right: when the principal curvature directions are not axis-aligned, diagonal rescaling cannot improve conditioning.}
    \label{fig:rescale}
\end{figure}
The symmetry's connection to diagonal preconditioning in \cref{sec:preconditioning} is significant. The condition number $\kappa(H)$ governs standard worst-case convergence bounds for first-order methods. Classical results \citep{forsythe1955best, SLUIS1970} show that diagonal preconditioning improves $\kappa$ when eigenvectors are not element-wise equal in magnitude. While no guarantee of universal improvement exists, diagonal preconditioning is highly effective when eigenvectors are approximately \emph{axis-aligned}. Appendix~\ref{sec:rescale} provides intuition via the simpler rescale symmetry, including a depiction of when diagonal rescaling helps and when it cannot. Adaptive optimizers such as Adam and Adagrad perform implicit diagonal preconditioning via per-coordinate learning rates.

In Appendix~\ref{app:transformers}, we find Hessian eigenvectors of minima after optimization are more axis-aligned than random and detail a symmetry in the Transformer architecture that allows for rotation in their MLP weights.
\subsection{Trace and average-case conditioning}
The worst-case condition number $\kappa(H) = \lambda_{\max}/\lambda_{\min}^+$ is determined by only two eigenvalues and can be dominated by outliers. A more robust measure considers the full spectrum. The ratio of the arithmetic to geometric mean of the positive eigenvalues captures \emph{average-case} conditioning and has been studied as an alternative to $\kappa$ in the context of diagonal preconditioning and iterative methods \citep{doan, jung2024omegaconditionnumberapplicationsoptimal, ghadimi2025newinsightsalgorithmsoptimal}.

While the geometric mean of nonzero eigenvalues (related to pseudo-determinant) is preserved under even splits (\cref{thm:pdet}), the \emph{arithmetic mean} can change. The trace $\tr(H)$ represents the sum of eigenvalues, and $\tr(H)/p$ (where $p$ is the number of positive eigenvalues) is the average curvature in non-symmetry directions.
\begin{figure}[t!]
    \centering
    \includegraphics[width=\textwidth]{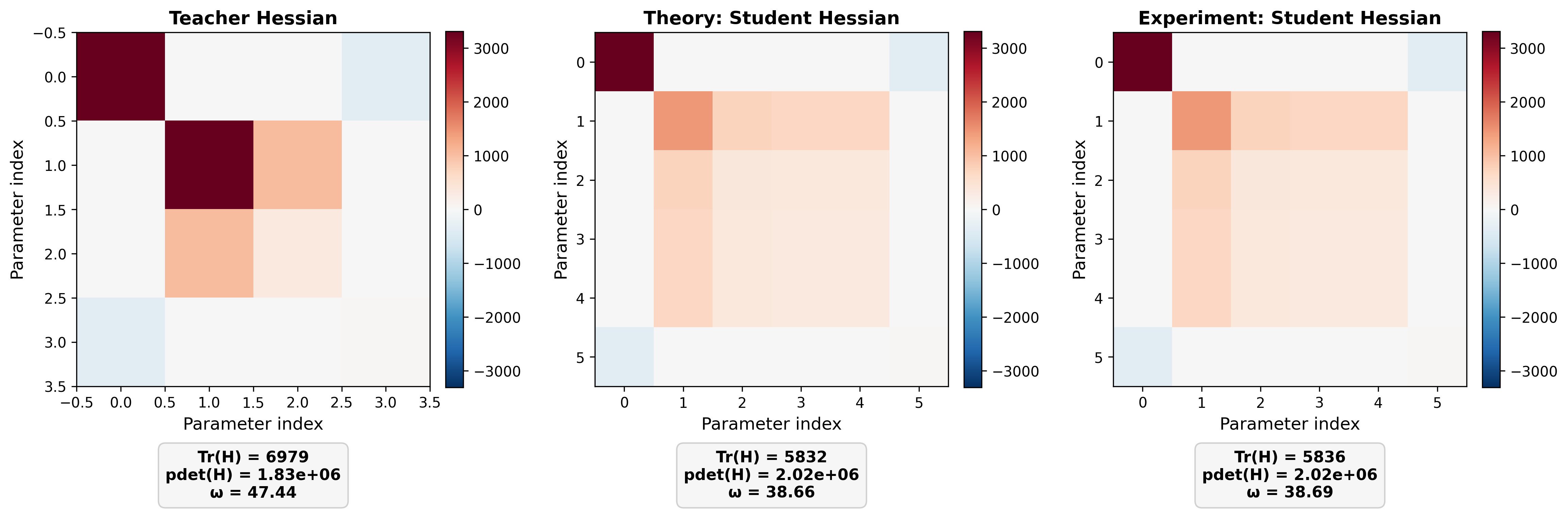}
    \caption{Theoretical and empirical Hessian comparison for a teacher-to-student expansion computed after optimizing to zero loss. \emph{Left:} teacher Hessian. \emph{Centre:} student Hessian predicted by \cref{thm:pdet}. \emph{Right:} empirically measured student Hessian after training to zero loss. The theoretical prediction matches the experiment to within 0.03\% in $\mathrm{Tr}(H)$, $\pdet(H)$, and $\omega$.}
    \label{fig:hessian_theory}
\end{figure}

\begin{theorem}[Trace Formula]
\label{thm:trace}
Let $H'$ be obtained by an order-$m$ even split of a neuron. Let $D$ be the duplicated weights and $S$ the split weights. Then:
\begin{equation}
\tr(H') = \tr(H_{\setminus S,D}) + m \sum_{s \in S} H_{ss} + \frac{1}{m} \sum_{d \in D} H_{dd},
\end{equation}
where $H_{\setminus S,D}$ is the matrix omitting all rows and columns corresponding to weights in $S \cup D$.
\end{theorem}

By treating the splitting degree $m$ as a continuous variable, we observe that the trace is a convex function of $m$. The new trace $\tr(H')$ remains strictly smaller than the original trace $\tr(H)$ for any degree $1 < m < \frac{\sum_{d \in D} H_{dd}}{\sum_{s \in S} H_{ss}}$, with the optimum $m^*$ derived in Appendix ~\ref{app:trace_proof}.

This suggests that although the volume is unchanged (for even splits), the \textbf{average curvature} in non-symmetry directions can be reduced and redistributed across dimensions.

\paragraph{Average-Case Condition Number.} We formalize the above using the AM-GM inequality for $p$ positive eigenvalues, $\frac{1}{p} \sum_{i=1}^p \lambda_i \geq \left(\prod_{i=1}^p \lambda_i\right)^{1/p}$, with equality iff all $\lambda_i$ are equal. The geometric mean $(\pdet H)^{1/p}$ is fixed under even splits. If we reduce the arithmetic mean $\tr(H)/p$, the eigenvalues become more uniform. The \emph{average-case conditioning ratio} is:
\begin{equation}
\omega(H) = \frac{\frac{1}{p} \sum_{i=1}^p \lambda_i}{\left(\prod_{i=1}^p \lambda_i\right)^{1/p}} = \frac{\tr(H)/p}{(\pdet H)^{1/p}}.
\end{equation}

This ratio equals 1 when all eigenvalues are equal (perfect conditioning) and increases with eigenvalue spread. Since the denominator is fixed under even splits while the numerator (trace) can decrease, overparameterization can improve average-case conditioning. We additionally note \citet{paquette2022homogenizationsgdhighdimensionsexact} show the average-case maximum learning rate goes as $\frac{2}{\tr(H)}$, suggesting overparameterization can allow for convergence with larger learning rates.

\section{Probability of Initializing Near Minima}
\label{sec:volume}
For a network of width $n$ expanded to width $n+1$ under Xavier initialization ($b_m = 1/\sqrt{m}$), the probability of initializing within $\epsilon$ of a global minimum satisfies
\begin{equation}
\label{eq:volume_ratio}
\frac{P_{n+1}}{P_{n}} \;\geq\;
\frac{\sqrt{2}\,\ln(1+\sqrt{2})}{8}
\cdot \|w^\mathrm{out}\|_1\,(n+1)^2 \left(\frac{n+1}{n}\right)^{\!n}
(\epsilon/C)^{1/4},
\end{equation}
where $\|w^\mathrm{out}\|_1 = \sum_{j=1}^n |w_j^\mathrm{out}|$ is the $\ell_1$ norm of the teacher's outgoing weights and $C > 0$ bounds the fourth-order coefficient of the loss along the locally flat direction introduced by splitting. The bound is derived for $d_\mathrm{in} = d_\mathrm{out} = 1$ in Appendix~\ref{app:volume_ratio}, with the higher-dimensional case following the same template.

% Intuitively, the bound combines an initialization-density penalty from the change in prior density with width, a $(n+1)$ permutation factor, an orbit arc-length contribution proportional to $\|w^\mathrm{out}\|_1$, and the $(\epsilon/C)^{1/4}$ width of the locally-flat sublevel set introduced by splitting. 
The right-hand side exceeds $1$ once $\|w^\mathrm{out}\|_1\,(n+1)^2 \gtrsim \epsilon^{-1/4}$, so overparameterization makes global minima strictly more reachable for sufficiently wide networks.
\begin{figure}[t!]
    \centering
    \includegraphics[width=\textwidth]{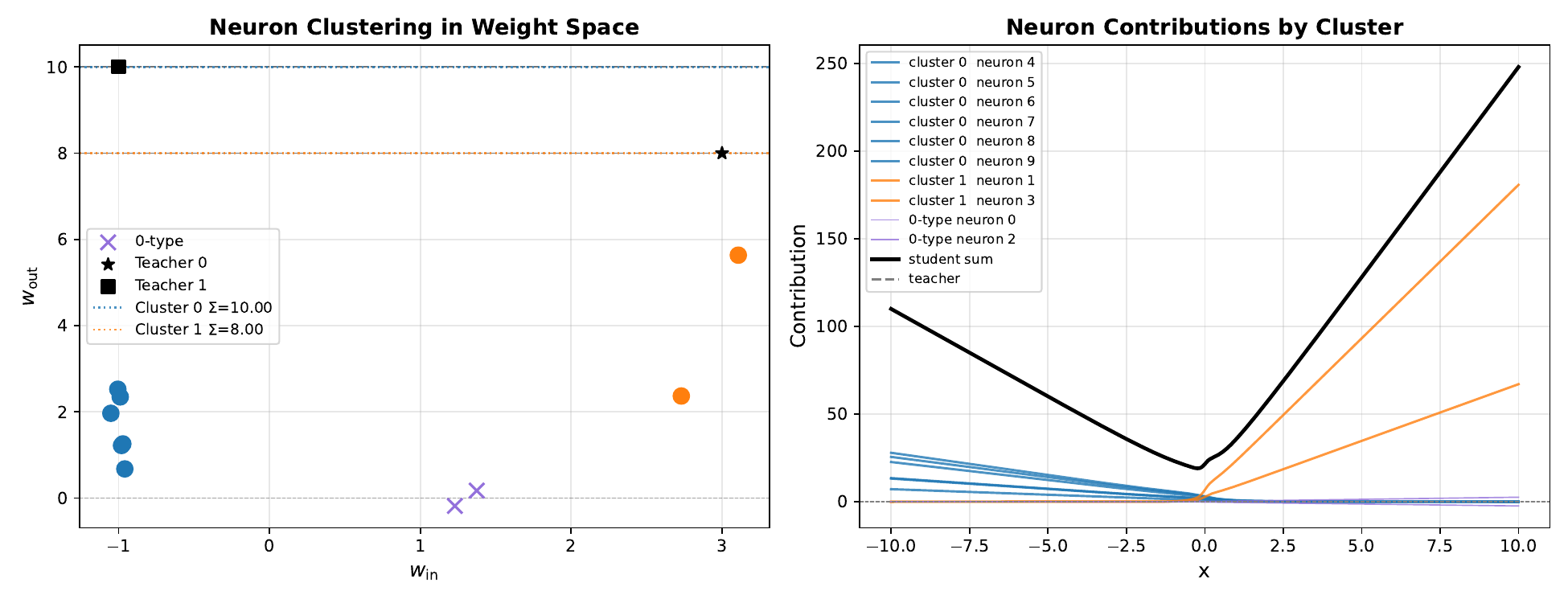}
    \caption{Neuron clustering at a global minimum of a width-10 student trained to match a width-2 teacher (zero loss via L-BFGS). \emph{Left:} student weights in $(w_{\mathrm{in}}, w_{\mathrm{out}})$ space. Neurons cluster into two groups, each with the same $w_{\mathrm{in}}$ as the corresponding teacher neuron and outgoing weights summing to the teacher's value ($\Sigma=10$ and $\Sigma=8$). Two neurons are 0-type (output weights cancel). \emph{Right:} per-cluster neuron contributions confirm each cluster implements one teacher neuron via splitting.}
    \label{fig:clusters}
\end{figure}

\section{Experiments}
\label{sec:experiments}
% We support our theoretical predictions with three families of experiments: small-network theory validation, where global minima are guaranteed, student-teacher experiments across MLP, CNN, and Transformer architectures, and direct data fitting without a teacher. In all cases, we train student networks of increasing width and compute Hessian properties at convergence.
We test our predictions in three settings: small networks with guaranteed convergence to global minima, student-teacher experiments on MLPs, CNNs, and Transformers and direct data fitting. 
% In each, we train students of increasing width and measure Hessian properties at convergence.
% Unless stated otherwise, we use $\mu$P for width-consistent learning rate scaling \cite{yang2022tensor}
% and report results averaged over 3 seeds
% . 
Full details and plots are in Appendix~\ref{app:experiments}.
\subsection{Predictions at Zero Loss}
\label{sec:theory_validation}
We directly test \cref{thm:pdet} by training small networks to zero loss with L-BFGS. To rule out global minima outside the symmetry orbit of the teacher, we use the activation $\sigma_{\alpha,\gamma}(x) = \sigma_{\mathrm{soft}}(x) + \alpha\,\sigma_{\mathrm{sig}}(\gamma x)$ from \citet{simsek2021geometry}, for which all global minima of the teacher--student loss are related by symmetry. \cref{fig:hessian_spectrum} shows the empirical eigenvalue spectrum of the student. \cref{fig:hessian_theory} shows the Hessian predicted by \cref{thm:pdet} matches the empirically measured Hessian to within $0.03\%$ in $\mathrm{Tr}(H)$, $\pdet(H)$, and $\omega$. Post-optimization weights exhibit the predicted structure (\cref{fig:clusters}; see also \cref{fig:neuron_contributions}).
% neurons cluster around teacher $w_{\mathrm{in}}$ values, with outgoing weights summing to the teacher's value and 0-type neurons filling the remainder.
\subsection{Student-Teacher Experiments}
\label{sec:mlp_scale}
\begin{figure}[t!]
    \centering
    \includegraphics[width=\textwidth]{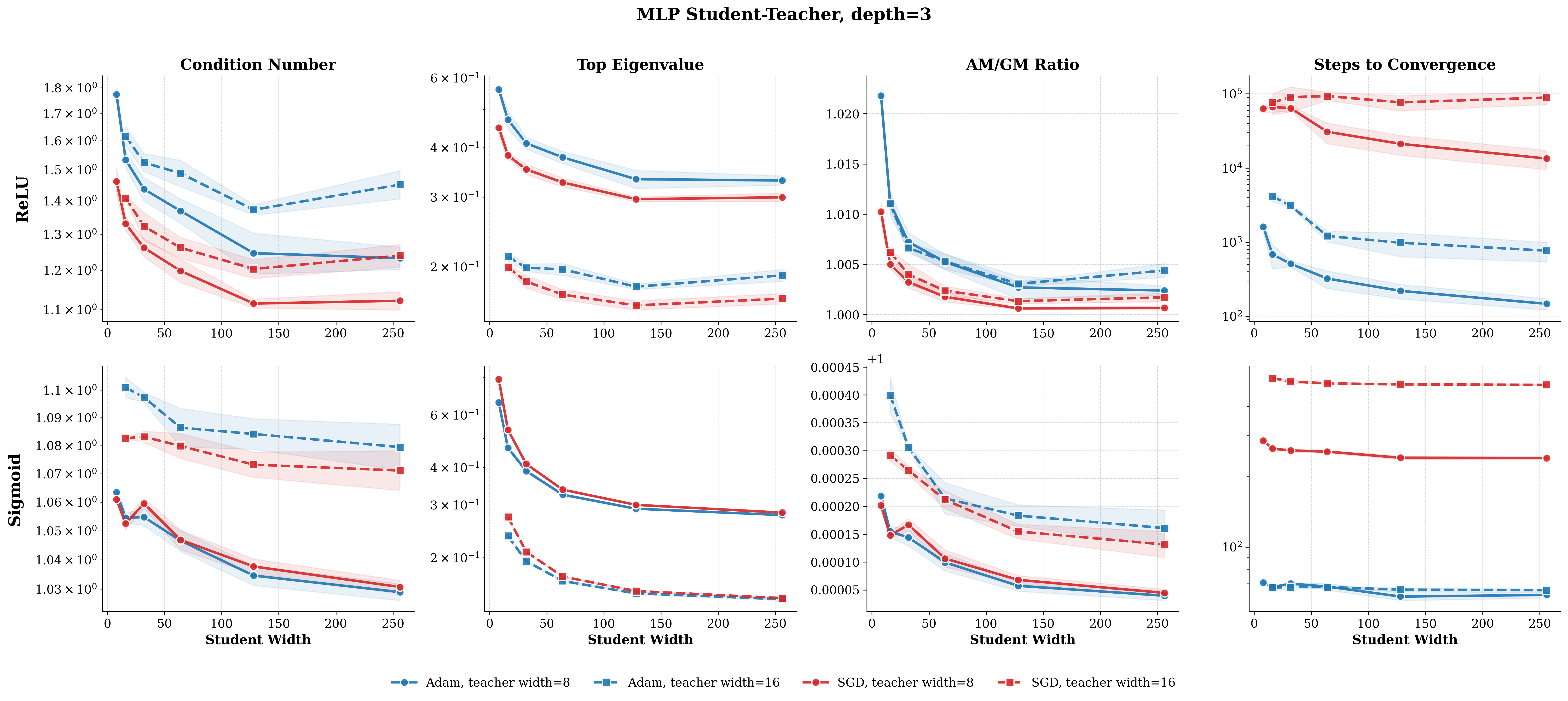}
    \caption{MLP student-teacher experiments (depth 3). Students use ReLU (top rows) and sigmoid (bottom rows), trained with Adam and SGD. Condition number, top eigenvalue, AM/GM ratio, and steps to convergence all improve monotonically with width.}
    \label{fig:mlp_cnn}
\end{figure}
These experiments test whether the Hessian improvements predicted by our theory persist at practical scale and across architectures. We train student networks of increasing width and compute Hessian properties at convergence. \cref{fig:mlp_cnn} shows MLP students trained with SGD and Adam: condition number, top eigenvalue, AM/GM ratio, and steps to convergence all improve monotonically with width. The same trends hold across depths (~\cref{fig:mlp_depth}). CNN and Transformer results (\cref{fig:cnn_lm}) show the same improvement; our theory applies directly to the MLP subspace of Transformer weights, which comprises the majority of parameters.
\FloatBarrier
\subsection{Direct Data Fitting}
\label{sec:broader}
To test whether the improvements depend on the teacher-student setup, we train networks directly on data. \cref{fig:data_fitting} shows MLPs on California Housing regression and ConvNets on CIFAR-10 classification; the model cannot fully overfit in the latter setting. Conditioning improves with width in both cases, which suggests our finding may continue to be relevant in the mildly underparameterized regime.
\begin{figure}[t]
    \centering
    \includegraphics[width=\textwidth]{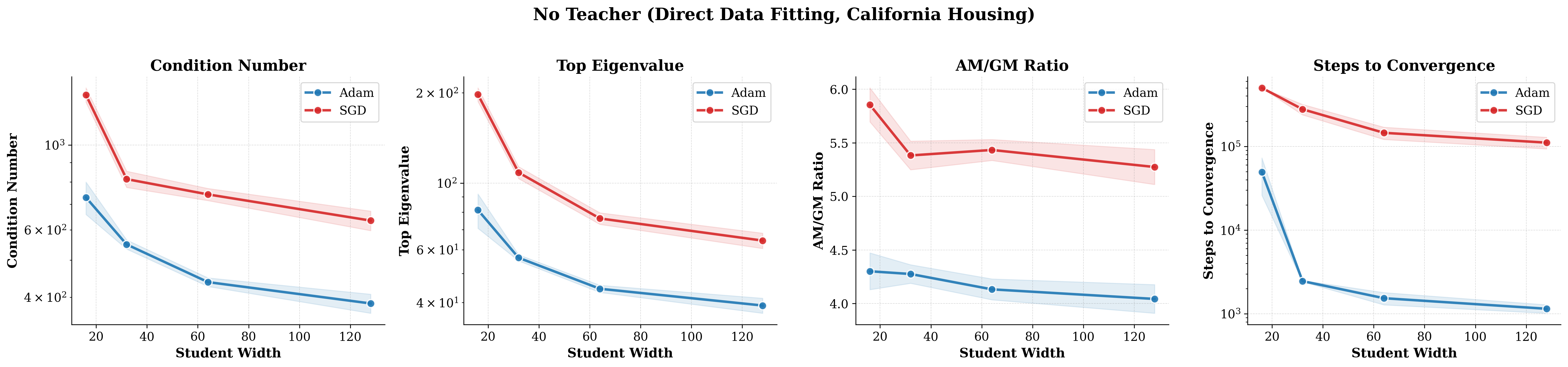}\\[4pt]
    \includegraphics[width=\textwidth]{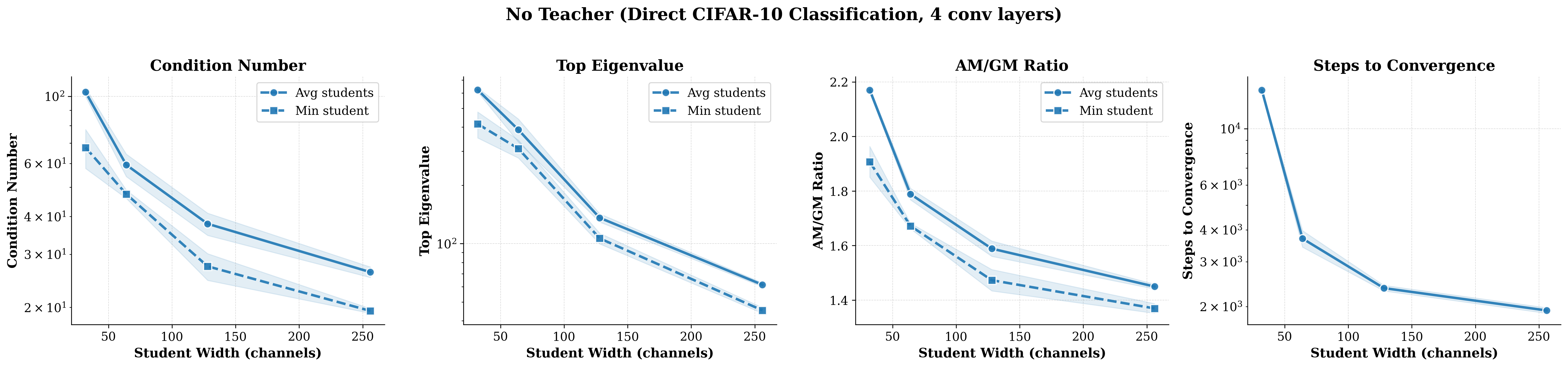}
    \caption{\emph{Top:} MLPs trained directly on California Housing regression (no teacher). \emph{Bottom:} ConvNets trained directly on CIFAR-10 classification (model cannot fully overfit). Conditioning improves with width in both settings.}
    \label{fig:data_fitting}
\end{figure}
\section{Discussion}
\label{sec:discussion}
\subsection{Implications for Optimization}
Our analysis is an existence result: we characterize which minima exist and their geometric properties, rather than proving which minima gradient descent converges to. In practical optimization, early stopping within a well-conditioned basin suffices and overparameterization makes such basins more abundant.
% This distinction is deliberate.\Cref{thm:trace} and \cref{thm:pdet} establish that overparameterization guarantees the existence of minima with improved conditioning and sufficient volume, while \cref{sec:volume} shows these minima occupy increasing probability mass near initialization.
Poorly-conditioned minima with small volume may also exist, but they are both harder to reach and harder to converge to, so optimization algorithms tend to naturally bypass them. Stochastic gradient noise provides additional selection pressure, destabilizing convergence in sharp, poorly-conditioned basins.

This is corroborated by \cref{fig:compare}, where the small gap between the mean and best student conditioning suggests that the existence of well-conditioned minima improves practical convergence. A possible connection to plasticity loss in continual learning is discussed in Appendix~\ref{app:plasticity}.
% This suggests a principle: \textbf{architecture enables, optimization selects}. Overparameterization expands the symmetry orbit to include well-conditioned solutions allowing the optimizer to converge to them preferentially.
\subsection{Implications for Generalization}
Our framework may also shed light on generalization. Consider noisy training data $\{x_i, y_i\}$ generated by some \textbf{simple} ground truth function $f$ with noise: $y_i = f(x_i) + \zeta$. The width required to overfit on $\{x_i, y_i\}$ is typically much larger than the width needed to represent $f$. As a result, the minima corresponding to $f$ have larger symmetry orbits and therefore greater total volume than those corresponding to overfitting solutions.
% This provides a counting argument for why overparameterized networks generalize: they converge to simpler functions not because they have lower loss, but because their minima are more numerous and easier to optimize.

Figure \ref{fig:data_fitting} Bottom provides a first look at direct data fitting and the underparameterized regime, where we can see the predicted trends continue to hold. In this setting, though we are not overparameterized, there exist sub-networks corresponding to regions of the weight space representing simpler functions (Figure \ref{fig:groupoid}).
% This is consistent with the observation of \citet{dinh2017sharp} that not all generalizing minima are wide; our analysis suggests that the minima we converge to \emph{tend to be} wide because of their favorable optimization properties.
% \subsection{Wideness of Minima and Anthropic Principle in Optimization}

There is an old idea in machine learning that generalizing minima are wide \citep{hochreiter1997flat}. However, considering parameter space symmetries it is clear this picture is incomplete \citep{dinh2017sharp}: using symmetry transformations we can construct narrow minima that generalize. A more plausible explanation is that the implication is in the other direction: wide minima represent simpler functions and simpler functions generalize. This suggests the wonderful coincidence of deep learning is that these generalizing minima also happen to be favored by our optimization process.
\subsection{Limitations and Future Work}
\label{sec:limitations}
Our analysis is restricted to width-wise symmetries. Extending the framework to depth-wise symmetries is a natural next direction. The relation $H' = B^\top H B$ also assumes zero-loss minima, so a perturbative analysis of the Hessian around small non-zero loss would extend our conditioning and volume results to the mildly underparameterized regime probed in \cref{fig:data_fitting}. Finally, our analysis is primarily about existence of good minima and we hope to connect our theory to practical training dynamics in future work.
\section{Conclusion}
We show that overparameterization benefits optimization through two distinct mechanisms: diagonal preconditioning of the Hessian and increased volume of minima near initialization. Our theoretical results are data and optimizer agnostic, depending only on the symmetry structure of the architecture which prescribes the geometry of the loss landscape.

\bibliographystyle{plainnat}
\bibliography{example_paper}

\newpage
%%%%%%%%%%%%%%%%%%%%%%%%%%%%%%%%%%%%%%%%%%%%%%%%%%%%%%%%%%%%%%%%%%%%%%%%%%%%%%%
%%%%%%%%%%%%%%%%%%%%%%%%%%%%%%%%%%%%%%%%%%%%%%%%%%%%%%%%%%%%%%%%%%%%%%%%%%%%%%%
\appendix

\section{Experimental Appendix}
\label{app:experiments}

\subsection{Architecture}

For the depth-1 sweep, networks are single-hidden-layer MLPs ($L=1$) with input and output dimension equal to the teacher width ($d_{\mathrm{in}} = d_{\mathrm{out}} = 8$). Deeper sweeps and other experiment families use the architectures listed in \cref{tab:hyperparams}. The teacher has hidden width $n_f = 8$; student widths are $n \in \{8, 16, 32, 64\}$, corresponding to overparameterization ratios $n/n_f \in \{1, 2, 4, 8\}$. We test both ReLU and sigmoid activations.

\subsection{Training}
\begin{figure}[t!]
    \centering
    \includegraphics[width=\textwidth]{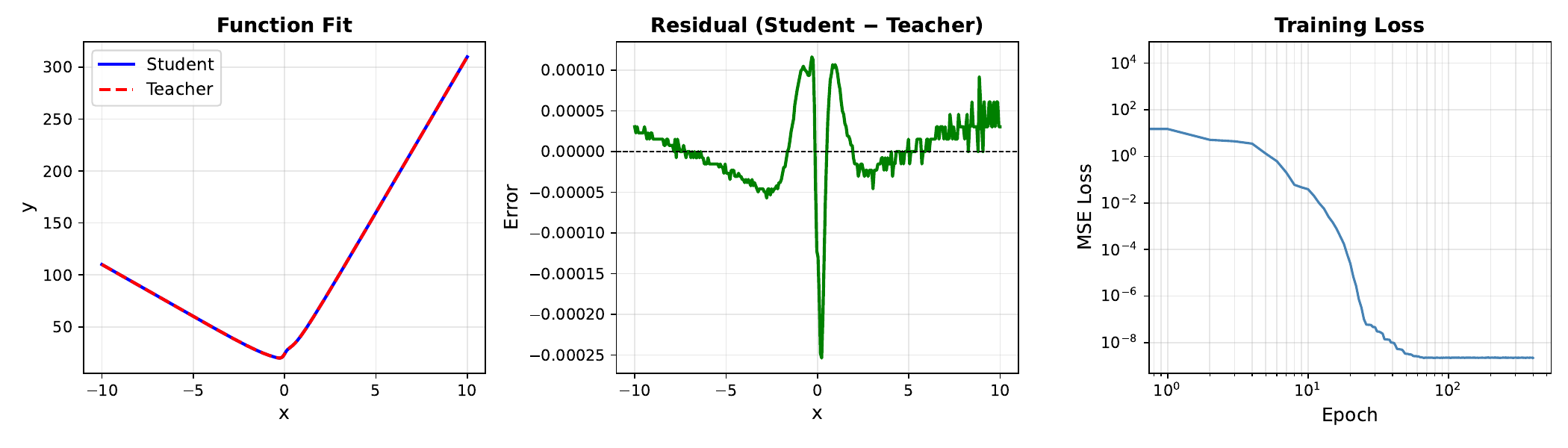}
    \caption{Training a width-3 student trained to match a width-2 teacher to zero loss via L-BFGS. \emph{Left:} Student aligns exactly with teacher. \emph{Centre:} There is a small residual error. \emph{Right:} Loss curve using L-BFGS.}
    \label{fig:neuron_contributions}
\end{figure}

\begin{figure}[t!]
    \centering
    \includegraphics[width=\textwidth]{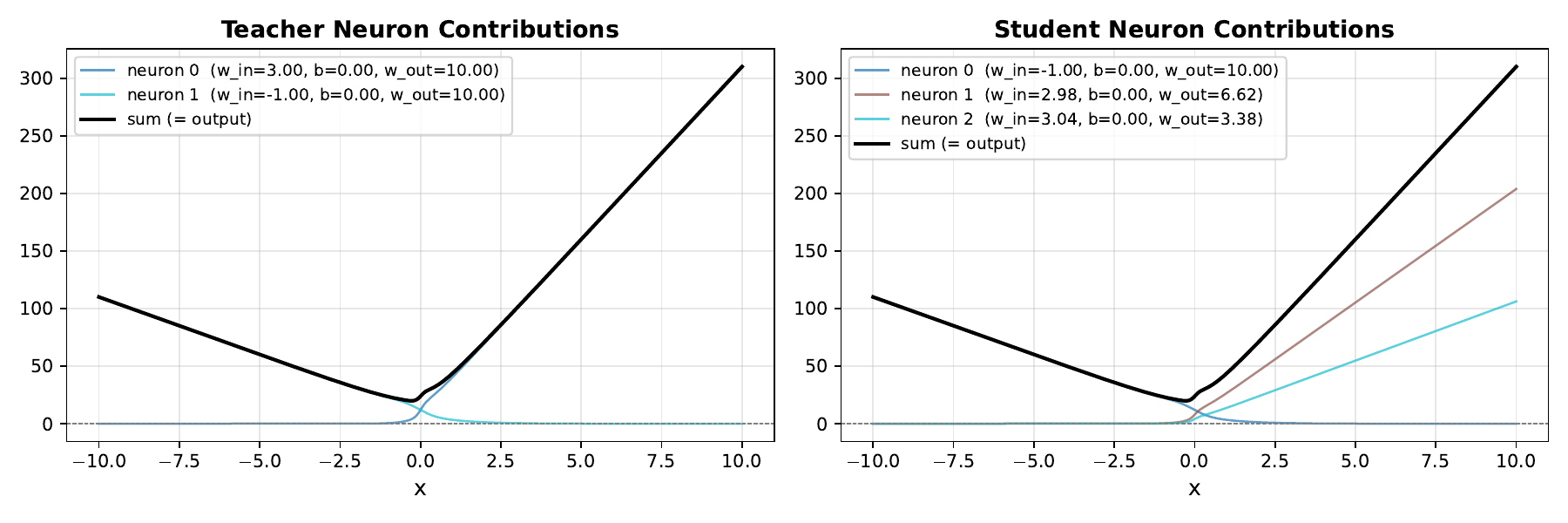}
    \caption{Weight structure at a global minimum of a width-3 student trained to match a width-2 teacher (zero loss via L-BFGS). \emph{Left:} teacher neuron contributions. \emph{Right:} student neuron contributions. Student neuron 0 replicates teacher neuron 1 exactly ($w_{\mathrm{in}}=-1.00$), while neurons 1 and 2 split teacher neuron 0 ($w_{\mathrm{in}}\approx 3.00$) with outgoing weights $6.62+3.38=10.00$, matching the predicted splitting structure.}
    \label{fig:neuron_contributions}
\end{figure}

\begin{figure}[t!]
    \centering
    \includegraphics[width=\textwidth]{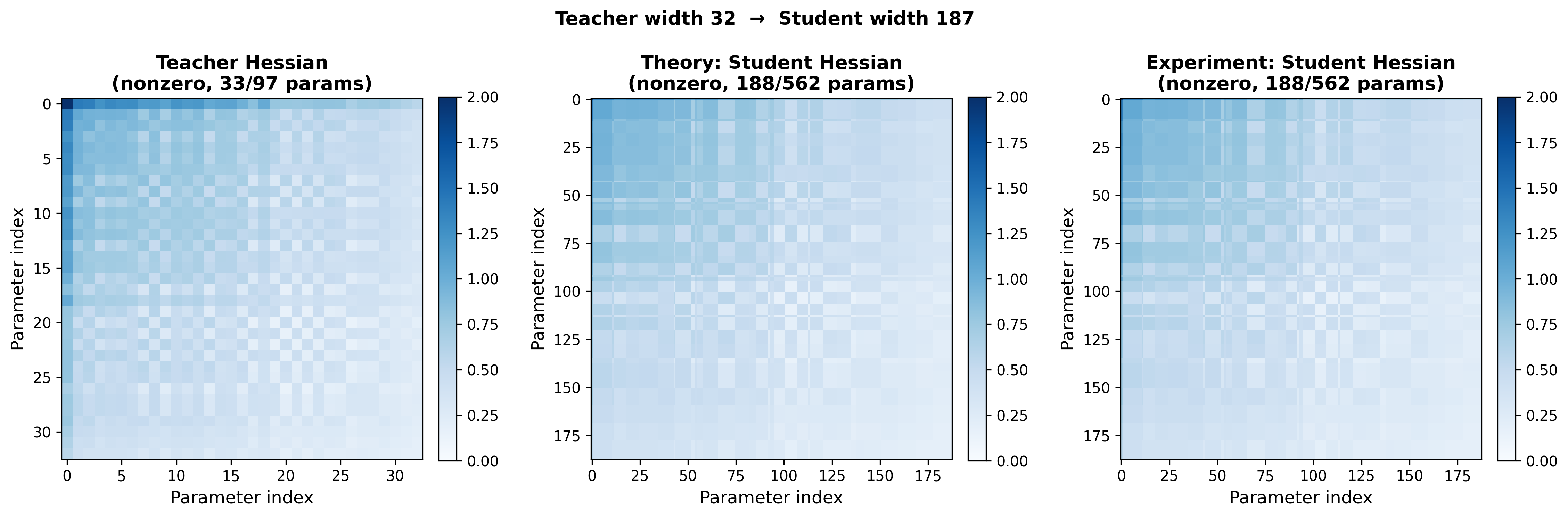}
    \caption{Theoretical and empirical Hessian comparison at larger scale computed by direct application of a symmetry transformation (teacher width 32 $\to$ student width 187, sigmoid activation). \emph{Left:} teacher Hessian. \emph{Centre:} student Hessian predicted by \cref{thm:pdet}. \emph{Right:} empirically measured student Hessian. Hessian entries agree nearly exactly, implying the prediction scales to larger networks.}
    \label{fig:hessian_theory_large}
\end{figure}

\begin{figure}[t!]
    \centering
    \includegraphics[width=\textwidth]{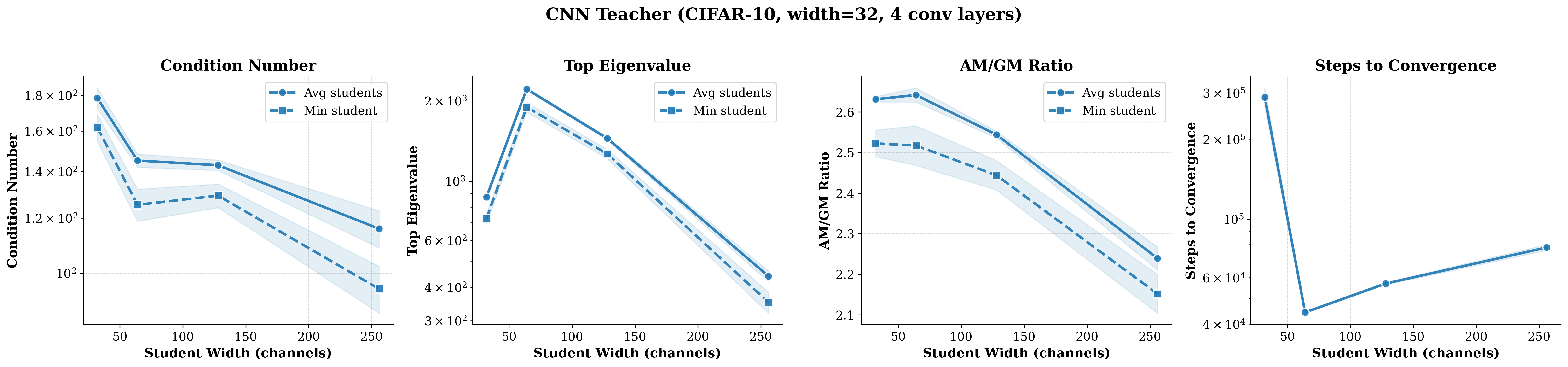}\\[6pt]
    \includegraphics[width=\textwidth]{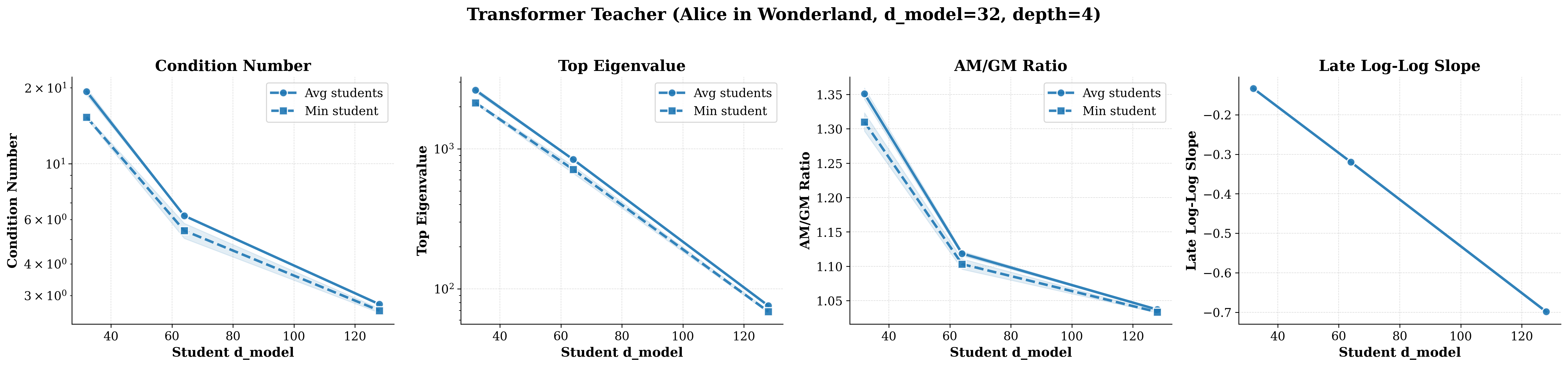}
    \caption{\emph{Top:} CNN student-teacher experiment. A ConvNet teacher (width 32, 4 conv layers) is pre-trained on CIFAR-10. Student ConvNets of increasing channel width are trained to match the teacher's output logits. \emph{Bottom:} Transformer student-teacher experiment. A character-level decoder-only Transformer teacher ($d_{\mathrm{model}}=32$, depth 4) is pre-trained on Alice in Wonderland. Student Transformers of increasing width are trained to match the teacher's output distribution via KL divergence. In both settings, condition number, top eigenvalue, and AM/GM ratio improve with width, alongside convergence behavior (steps to convergence for the CNN, late log-log slope for the Transformer).}
    \label{fig:cnn_lm}
\end{figure}

\begin{figure}[t!]
    \centering
    \includegraphics[width=0.5\textwidth]{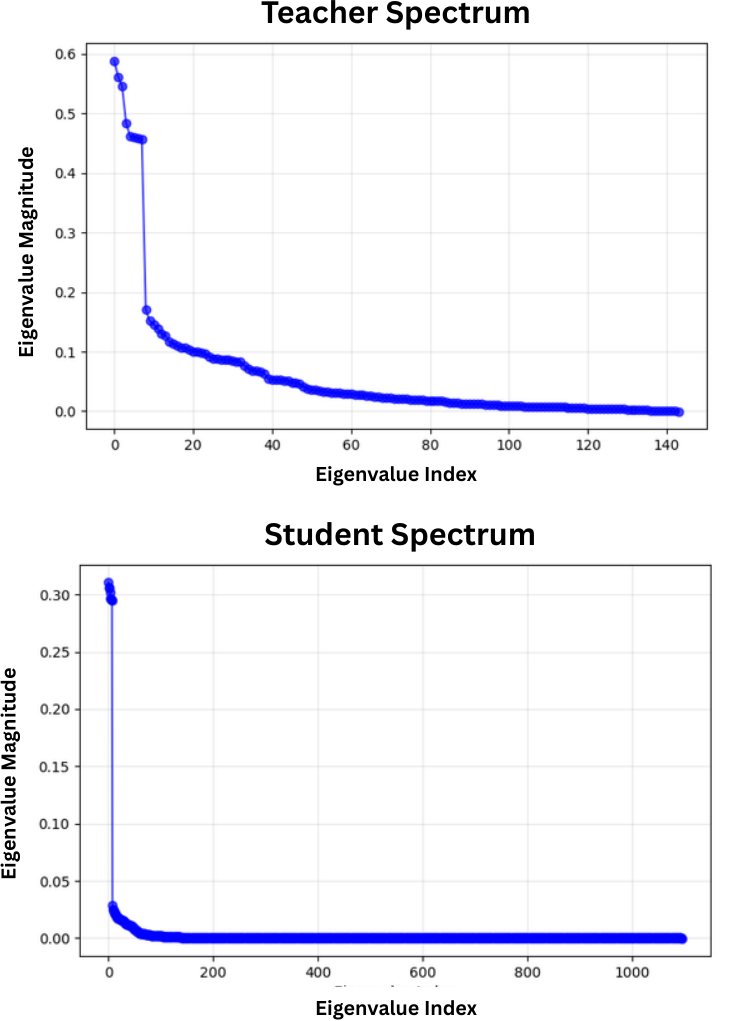}
    \caption{Hessian eigenvalue spectrum for a teacher network (top, width $n_f = 8$) and an overparameterized student (bottom, width $n = 64$). The teacher spectrum decays smoothly, while the student spectrum concentrates its curvature in a small number of leading eigenvalues with the vast majority near zero, corresponding to the flat directions introduced by overparameterization symmetries.}
    \label{fig:spectrum}
\end{figure}

\begin{figure}[t!]
    \centering
    \includegraphics[width=\textwidth]{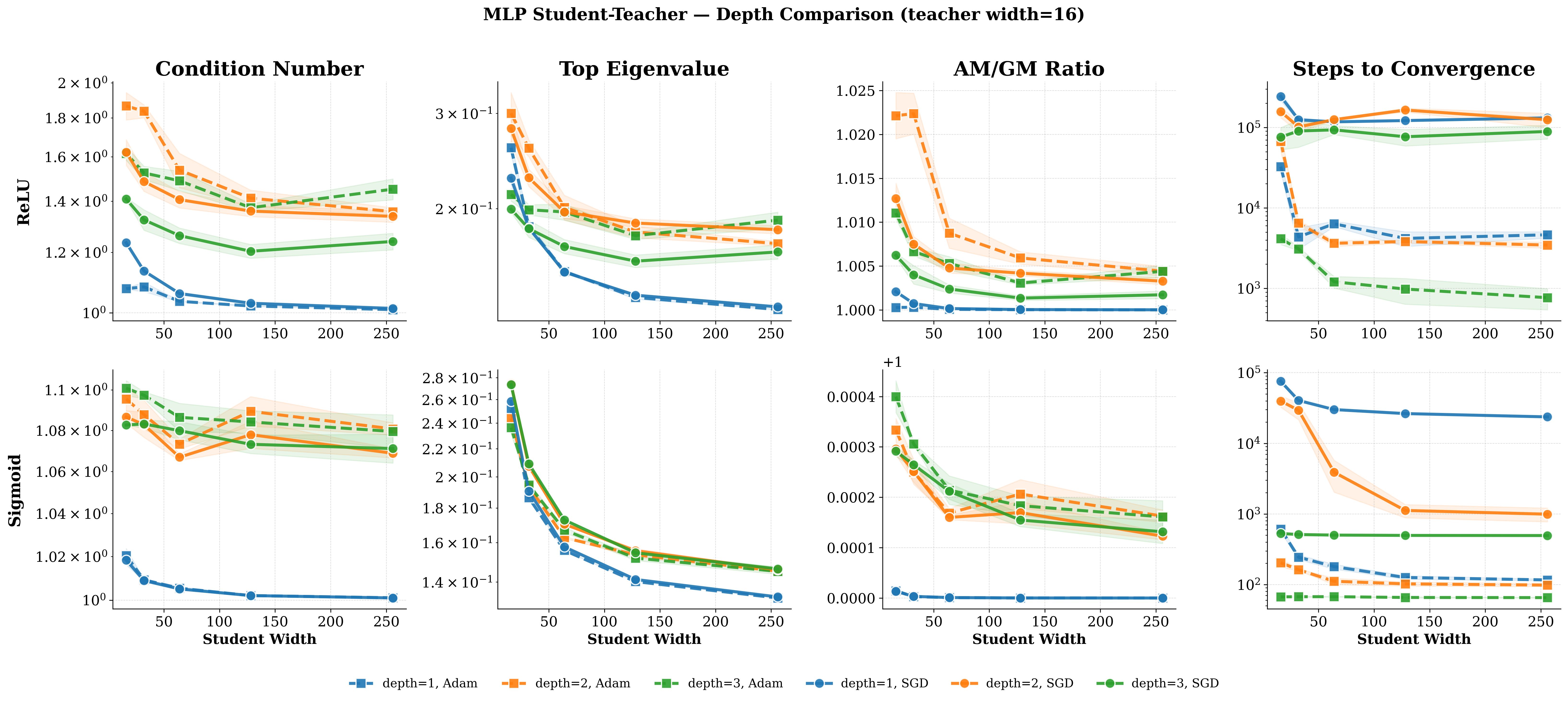}
    \caption{Effect of depth on MLP student-teacher conditioning (teacher width 16). Students at depth 1, 2, and 3 are trained with Adam and SGD (ReLU top, sigmoid bottom), with student width up to 256. Conditioning improves with width at all depths with no clear depth dependence, suggesting the symmetry-based mechanism is robust to depth.}
    \label{fig:mlp_depth}
\end{figure}

\begin{figure}[t!]
    \centering
    \includegraphics[width=\textwidth]{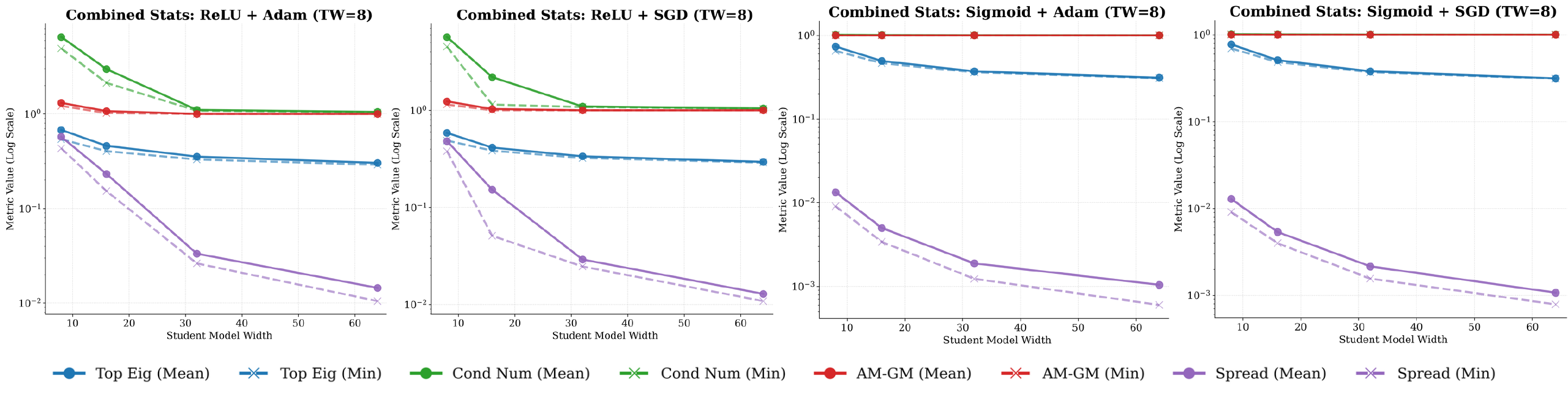}
    \caption{A comparison of conditioning across the average and best conditioned students. Training is done with 1 teacher and 50 student models across 5 seeds (an extended sweep, distinct from the main MLP setup in \cref{tab:hyperparams}).}
    \label{fig:compare}
\end{figure}

\begin{figure}[t!]
    \centering
    \includegraphics[width=\textwidth]{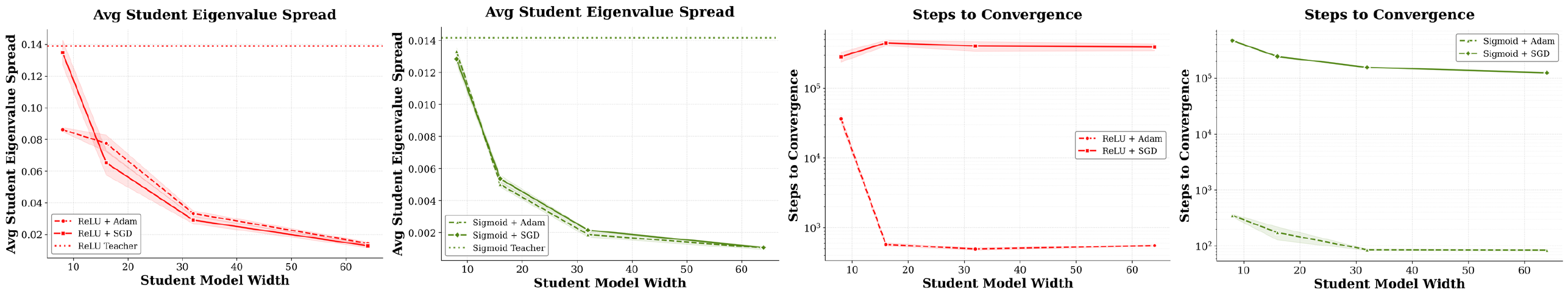}
    \caption{Additional Plots for Teacher-Student Network Training. Eigenvalue spread and steps to convergence.}
    \label{fig:additional}
\end{figure}

\paragraph{Synthetic MLP student-teacher.}
For each configuration, we train 5 independent student networks in parallel
to match the frozen teacher's outputs. At each training step, a fresh batch
of inputs is sampled from $\mathcal{N}(0, I)$ and the teacher generates
targets; each student sees the same batch. The loss is the mean squared
error between student and teacher outputs.

We use maximal update parameterization $\mu$P \cite{yang2022tensor} for
width-consistent learning dynamics, with base width 8. We test two
optimizers:
\begin{itemize}
    \item \textbf{SGD}: learning rate $0.1$, maximum 50{,}000 epochs.
    \item \textbf{Adam}: learning rate $0.01$, maximum 5{,}000 epochs.
\end{itemize}
Training terminates early if the loss falls below $5 \times 10^{-5}$ or if
no improvement is observed for 20{,}000 consecutive epochs.

\paragraph{MLP on California Housing.}
As a no-teacher counterpart to the synthetic experiment, students are trained directly to predict median house value from the 8 standardised California Housing features, on a fixed 200-sample training subset. All other hyperparameters --- $\mu$P, optimizer pair, learning rates, epoch budgets, early-stopping --- match the synthetic experiment.

\paragraph{CNN on CIFAR-10.}
The architecture is a 4-block ConvNet (Conv $\to$ BN $\to$ ReLU, with
$2{\times}2$ MaxPool every two blocks) followed by global average pooling
and a linear head. Channel widths $\{32, 64, 128, 256\}$ are swept. Two
settings share hyperparameters except for the loss: \emph{teacher--student},
matching the logits of a pretrained width-32 teacher under MSE
(early-stop $5 \times 10^{-3}$), and \emph{direct}, predicting CIFAR-10
labels under cross-entropy (early-stop $0.4$). Both use Adam at $10^{-3}$,
batch size 256, maximum 5{,}000 epochs.

\paragraph{Character-level transformer on Alice in Wonderland.}
The teacher is a 4-layer, 4-head encoder-only transformer with
$d_{\text{model}}{=}32$, learned absolute positional embeddings, and a
73-character vocabulary, pretrained on the Alice corpus tokenised into
128-character windows. Students share the architecture but vary
$d_{\text{model}} \in \{32, 64, 128\}$ and minimise per-token KL divergence
to the teacher's logits. We use Adam at $10^{-3}$, batch size 256, up to
50{,}000 epochs, and compute the top 50 Hessian eigenvalues (rather than
100) per student.

\subsection{Compute Resources}

All experiments are run on NVIDIA L40S GPUs. A single seed of a single configuration takes between roughly 5 minutes (smallest synthetic MLP runs) and 12 hours (largest Transformer student-teacher runs).

\subsection{Hessian Computation}

After convergence, we compute the Hessian eigenspectrum using the Lanczos algorithm with full reorthogonalization, which efficiently approximates the eigenvalue distribution via Hessian-vector products without forming the full matrix. For small networks used in theory validation we additionally compute the full Hessian exactly via finite differences to verify the Lanczos estimates. We extract: the complete eigenvalue spectrum, trace, top eigenvalue, condition number ($\lambda_{\max} / \lambda_{\min}^+$), pseudo-determinant (product of nonzero eigenvalues), and the AM/GM ratio of positive eigenvalues. To handle Lanczos numerical noise, eigenvalues below a small magnitude threshold are set to zero. Gradient norms are also recorded at the final iterate.
\subsection{Hyperparameters}

\begin{table}[h]
\centering
\small
\setlength{\tabcolsep}{4pt}
\begin{tabular}{lcccc}
\toprule
\textbf{Parameter} & \textbf{MLP s--t} & \textbf{MLP Housing} & \textbf{CNN}\textsuperscript{$*$} & \textbf{LM s--t} \\
\midrule
Architecture            & MLP                            & MLP            & ConvNet (4 conv) & CharTransformer \\
Dataset                 & synthetic                      & Cal.\ Housing  & CIFAR-10         & Alice in Wonderland \\
Teacher                 & random init                    & none (direct)  & pretrained       & pretrained \\
Teacher width $n_f$     & $\{8,16,32,64\}^{\dagger}$     & 16             & 32 channels      & 32 ($d_{\text{model}}$) \\
Student widths $n$      & $n_f\!\cdot\!\{1,2,4,8,16\}$   & $\{16,32,64,128\}$ & $\{32,64,128,256\}$ & $\{32,64,128\}$ \\
Depth                   & $\{1,2,3\}^{\dagger}$          & 2              & 4                & 4 \\
Input / output dim.     & $n_f / n_f$                    & 8 / 1          & $3{\times}32^2$ / 10 & vocab=73 \\
Activation              & ReLU, sigmoid                  & ReLU           & ReLU             & ReLU FFN, 4 heads \\
\# students             & 5                              & 25             & 5                & 5 \\
Optimizer(s)            & SGD, Adam                      & SGD, Adam      & Adam             & Adam \\
LR (SGD / Adam)         & 0.1 / 0.01                     & 0.1 / 0.01     & $10^{-3}$        & $10^{-3}$ \\
Max epochs (SGD/Adam)   & 50k / 5k                       & 50k / 5k       & 5k               & 50k \\
Loss                    & MSE                            & MSE            & MSE / CE\textsuperscript{$*$} & KL on logits \\
Batch size              & online                         & 200            & 256              & 256 \\
Early-stop loss         & $5{\times}10^{-5}$             & $5{\times}10^{-5}$ & $0.005$/$0.4$\textsuperscript{$*$} & --- \\
Patience (epochs)       & 20k                            & 20k            & 20k              & 20k \\
Param.                  & $\mu$P (base 8)                & $\mu$P (base 8)& standard         & standard \\
Top-$k$ Hessian eigs.   & 100                            & 100            & 100              & 50 \\
Seeds                   & 0, 1, 2                        & 0, 1, 2        & 0, 1, 2          & 0, 1, 2 \\
\bottomrule
\end{tabular}
\caption{Experimental hyperparameters.
\textsuperscript{$\dagger$} For the synthetic MLP student-teacher, depth 1 is
swept over $n_f \in \{8,16,32,64\}$ and depths 2--3 over $n_f \in \{8,16\}$.
\textsuperscript{$*$} CNN sweep covers two settings sharing all
hyperparameters except loss/early-stop: \emph{teacher--student} (MSE vs.\
pretrained ConvNet, stop $0.005$) and \emph{direct} (cross-entropy vs.\
CIFAR-10 labels, stop $0.4$). MLP Housing uses 200 standardized real samples;
CNN uses the full CIFAR-10 training split; LM uses the Alice corpus tokenised
into 128-character non-overlapping windows.}
\label{tab:hyperparams}
\end{table}

\section{Proof Details}
\label{app:proofs}

\subsection{Proof of \cref{thm:gradient} (Gradient Scaling)}
    Using the standard notation
    \[
        a^{(l)} = w^{(l)} h^{(l-1)} + b^{(l)}, \qquad
        h^{(l)} = \sigma(a^{(l)})
    \]
    for the pre-activations and activations at layer $l$. Error backpropagates as

    $$\delta^{(l-1)} = \sigma' \circ (W^l)^T \cdot \delta^{(l)}.$$
    The gradient of the weights are written
    \[
        \frac{\partial L}{\partial w^{(l)}_{ji}} = \delta^{(l)}_j \, h^{(l-1)}_i.
    \]

    For post-activation splitting, we duplicate the activations in layer $l$ and split the outgoing weights in layer $l+1$. The gradient of the outgoing split weights is unaffected since $\delta^{(l+1)}_j$ and $h^{(l)}_i$ are fixed. The gradient of the incoming duplicated weights in layer $l$ is divided by $m$ since $\delta^{(l)}_j$ is split.

    For pre-activation splitting, we split the activations in layer $l$ and duplicate the outgoing weights in layer $l+1$. The gradient of the incoming split weights is unaffected since $\delta^{(l)}_j$ and $h^{(l-1)}_i$ are fixed.

    The gradient of the outgoing duplicated weights in layer $l+1$ is divided by $m$ since $h^{(l)}_i$ is split.

    Finally, we can see downstream gradients are not impacted since $M^T\cdot \sigma' \cdot M^{+T} = \sigma'$:
    $$\delta^{(l-1)} = \sigma' \circ (MW^l)^T \cdot \sigma' \circ (W^{l+1}M^+)^T \cdot \delta^{(l+1)} = \sigma' \circ (W^l)^T \cdot \sigma' \circ (W^{l+1})^T \cdot \delta^{(l+1)}.$$

\subsection{Proof of \cref{thm:precond} (Diagonal Preconditioning)}

Applying Sylvester's Identity and letting $K = BB^\top \in \mathbb{R}^{a \times a}$, the characteristic polynomial of $H'$ is
\begin{equation}
\label{eqn:det_eig}
\det(H' - \lambda I_b) = (-\lambda)^{b-a}\det(HK - \lambda I_a) = (-\lambda)^{b-a}\det(K^{1/2} H K^{1/2} - \lambda I_a)
\end{equation}
by similarity, since $K$ is positive-definite for a valid expansion. This shows that $H'$ has exactly $b-a$ zero eigenvalues, and its remaining $a$ nonzero eigenvalues are identical to those of the full-rank $a \times a$ matrix $HK$.

\subsection{Proof of \cref{thm:pdet} and Associated Theorems}

To evaluate the change in the pseudo-determinant under expansion, we first explicitly characterize the diagonal matrix $K = BB^\top \in \mathbb{R}^{a \times a}$. As established in the main text, the expanded Hessian is $H' = B^\top H B$, and the nonzero eigenvalues of $H'$ match those of $HK$.

The matrix $B \in \mathbb{R}^{a \times b}$ maps the gradients of the $b$ expanded parameters back to the $a$ original parameters. Because $K = BB^\top$, the diagonal entries of $K$ correspond exactly to the inner products of the rows of $B$:
\begin{itemize}
    \item \textbf{Untouched weights:} The gradient mapping is a standard basis vector. The inner product is $1$, yielding a diagonal entry of $1$.
    \item \textbf{Split weights ($s \in S$):} a weight whose split copies have values $\alpha_i$ times the original (with $\sum_i \alpha_i = 1$). By \cref{thm:gradient}, each copy has \emph{unchanged} gradient, so the row in $B$ contains $m$ ones, yielding a diagonal entry of $m$.
    \item \textbf{Duplicated weights ($d \in D$):} a weight whose copies are \emph{identical} to the original, but whose copy gradients are scaled by the splitting proportions $\alpha_{d,i}$ (\cref{thm:gradient}). The row in $B$ contains the $\alpha_{d,i}$ values, yielding a diagonal entry of $\sum_{i=1}^m \alpha_{d,i}^2$.
\end{itemize}

Because $K$ is a diagonal matrix, its determinant is simply the product of these diagonal entries. Assuming the expansion involves an equal number of split and duplicated weights ($|S| = |D|$), we can pair the scaling factors of each split weight with a duplicated weight. The determinant simplifies to:
$$
\det(K) = 1^{|O|} \cdot \prod_{s \in S} m \cdot \prod_{d \in D} \left( \sum_{i=1}^m \alpha_{d, i}^2 \right) = \prod_{d \in D} \left( m \sum_{i=1}^m \alpha_{d, i}^2 \right)
$$
More generally, the first form $\det(K) = m^{|S|} \prod_{d \in D}(\sum_i \alpha_{d,i}^2)$ gives $\det(K) = m^{|S| - |D|}$ under even splitting, so volume is preserved only when $|S| = |D|$.

With the scaling factor $\det(K)$ explicitly defined, we now formally prove the invariance under even splitting and the volume expansion under uneven splitting.

\begin{theorem}\label{thm:pdet_invar}
    The pseudo-determinant of a full-rank Hessian $H$ is invariant under even-splitting transformations, assuming an equal number of split and duplicated weights ($|S|=|D|$).
\end{theorem}

\begin{proof}
    For an even split, an original parameter is divided uniformly such that $\alpha_{d,i} = 1/m$ for all $i = 1, \dots, m$. We evaluate the sum of the squared coefficients for any duplicated parameter $d \in D$:
    $$
    \sum_{i=1}^m \alpha_{d, i}^2 = \sum_{i=1}^m \left(\frac{1}{m}\right)^2 = m \left(\frac{1}{m^2}\right) = \frac{1}{m}
    $$
    Substituting this result into our expression for $\det(K)$ from the general framework, we find:
    $$
    \det(K) = \prod_{d \in D} \left( m \cdot \frac{1}{m} \right) = \prod_{d \in D} (1) = 1
    $$
    Because $\det(K) = 1$, the pseudo-determinant is strictly conserved:
    $$
    \pdet(H') = \pdet(H)\det(K) = \pdet(H)
    $$
    This completes the proof.
\end{proof}

\begin{theorem}\label{thm:uneven_volume}
    Let $H'$ be obtained after acting on the weights corresponding to a Hessian $H$ with the pre-activation or post-activation splitting groupoid. For any general split where $|S|=|D|$, the pseudo-determinant of the Hessian satisfies $\pdet(H') \ge \pdet(H)$.
\end{theorem}

\begin{proof}
    Consider a general split where the proportions $\alpha_{d, 1}, \dots, \alpha_{d, m}$ are not necessarily uniform, but inherently satisfy $\sum_{i=1}^m \alpha_{d, i} = 1$. By the Root-Mean-Square/Arithmetic Mean (RMS-AM) inequality, the sum of the squares of these coefficients is strictly bounded below:
    $$
    \sqrt{\frac{1}{m} \sum_{i=1}^m \alpha_{d, i}^2} \ge \frac{1}{m} \sum_{i=1}^m \alpha_{d, i} = \frac{1}{m}(1)
    $$
    Squaring both sides and multiplying by $m$, we obtain:
    $$
    \sum_{i=1}^m \alpha_{d, i}^2 \ge \frac{1}{m} \implies m \sum_{i=1}^m \alpha_{d, i}^2 \ge 1
    $$
    Applying this inequality to our established expression for $\det(K)$, we find:
    $$
    \det(K) = \prod_{d \in D} \left( m \sum_{i=1}^m \alpha_{d, i}^2 \right) \ge \prod_{d \in D} (1) = 1
    $$
    Consequently, the scaling factor applied to the pseudo-determinant is always greater than or equal to $1$, yielding:
    $$
    \pdet(H') = \pdet(H)\det(K) \ge \pdet(H)
    $$
    By the exact equality conditions of the RMS-AM inequality, $\det(K) = 1$ if and only if $\alpha_{d, 1} = \dots = \alpha_{d, m}$ for all $d \in D$. Thus, the pseudo-determinant strictly increases under any uneven split, recovering equality only in the case of a perfectly even split as proven in \autoref{thm:pdet_invar}.
\end{proof}
\subsection{Proof of \cref{thm:trace} (Trace Formula)}
\label{app:trace_proof}

Direct calculation from the Hessian block structure. After an order-$m$ split, the diagonal entries transform as:
\begin{itemize}
    \item Entries in $H_{OO}$: unchanged
    \item Entries in $H_{SS}$: each entry $H_{ss}$ is replicated $m$ times, contributing $m \cdot H_{ss}$
    \item Entries in $H_{DD}$: each entry $H_{dd}$ is replicated $m$ times but scaled by $1/m^2$, contributing $(1/m) \cdot H_{dd}$
\end{itemize}

Summing: $\tr(H') = \tr(H_{OO}) + m \sum_s H_{ss} + (1/m) \sum_d H_{dd}$.

Since originally $\tr(H) = \tr(H_{OO}) + \sum_s H_{ss} + \sum_d H_{dd}$, we get the stated formula.

\paragraph{Extension to non-even splits.}
The same calculation (or directly via $\tr(H') = \tr(B^\top H B) = \tr(KH)$, using the diagonal entries of $K$ from \cref{thm:pdet}'s proof) extends to splits with arbitrary coefficients $\alpha_1, \ldots, \alpha_m$ summing to 1:
\begin{equation}
\tr(H') = \tr(H_{\setminus S,D}) + m \sum_{s \in S} H_{ss} + \left(\sum_{i=1}^m \alpha_i^2\right) \sum_{d \in D} H_{dd},
\end{equation}
recovering \cref{thm:trace} when $\alpha_i = 1/m$. Since RMS-AM gives $\sum_i \alpha_i^2 \geq 1/m$ with equality iff $\alpha_i = 1/m$, even splits also minimize the trace at fixed $m$.

\paragraph{Optimal Splitting Degree.}
Treating the splitting degree $m$ as continuous, the trace formula is convex in $m$; we call its minimizer (the value of $m$ that yields the smallest trace under the fixed split structure $S, D$) the \emph{optimal splitting degree}. Setting $\partial \tr(H') / \partial m = 0$,
\begin{equation}
m^* = \sqrt{\frac{\sum_{d \in D} H_{dd}}{\sum_{s \in S} H_{ss}}}.
\end{equation}
Substituting $m^*$ back into the trace formula (or applying AM-GM) yields the minimum trace
\begin{equation}
\min \tr(H') = \tr(H_{\setminus S,D}) + 2\sqrt{\sum_{d \in D} H_{dd} \sum_{s \in S} H_{ss}}.
\end{equation}

\section{Zero-type Neurons}
\label{app:zero_type}
The main text describes two function-preserving operations for arbitrary activations. The first (neuron duplication) is characterized in \cref{sec:general} using duplication matrices $M$ and pseudoinverses $M^+$. Here we give an analogous characterization of the second: adding \emph{0-type neurons}, whose total contribution to the network output cancels.

\paragraph{Setup.} Consider a single hidden layer with $n$ neurons, input weight matrix $W_l \in \R^{n \times d_{\mathrm{in}}}$ and output weight matrix $W_{l+1} \in \R^{d_{\mathrm{out}} \times n}$. We expand to $n + k$ neurons via expansion matrix $M_0 \in \R^{(n+k) \times n}$ and contraction matrix $M_0^+ \in \R^{n \times (n+k)}$:
\begin{equation}
M_0 = \begin{bmatrix} I_n \\ A \end{bmatrix}, \qquad M_0^+ = \begin{bmatrix} I_n & Z \end{bmatrix},
\end{equation}
where $A \in \R^{k \times n}$ and $Z \in \R^{n \times k}$. The transformation $W_l \to M_0 W_l$ and $W_{l+1} \to W_{l+1} M_0^+$ adds $k$ new neurons with input weights $A W_l$ and effective output weights $W_{l+1} Z$.

\paragraph{Function preservation.} Verifying $f_\theta(x)$ is preserved:
\begin{align}
W_{l+1} M_0^+ \sigma(M_0 W_l x) &= \begin{bmatrix} W_{l+1} & W_{l+1} Z \end{bmatrix} \begin{bmatrix} \sigma(W_l x) \\ \sigma(A W_l x) \end{bmatrix} \nonumber \\
&= W_{l+1}\, \sigma(W_l x) + W_{l+1} Z\, \sigma(A W_l x).
\end{align}
The new neurons' total contribution $W_{l+1} Z\, \sigma(A W_l x)$ must vanish for all $x$. Partition the new neurons into groups by equality of their input weights (rows of $A W_l$). For arbitrary $\sigma$, this constraint is equivalent to: for each group $G_g$, the corresponding columns of $W_{l+1} Z$ sum to the zero vector. Writing $\mathbf{1}_g \in \R^k$ for the indicator vector of group $G_g$:
\begin{equation}
W_{l+1} Z\, \mathbf{1}_g = 0 \quad \text{for every group } g.
\end{equation}

\paragraph{Two special cases.}
\begin{itemize}
    \item \textbf{Identically zero outputs:} when no two new neurons share an input weight (rows of $A W_l$ are distinct), each group has size 1 and the constraint reduces to $W_{l+1} Z = \mathbf{0}$, i.e., every column of $Z$ lies in the right null space of $W_{l+1}$.
    \item \textbf{Cancellation:} when new neurons come in pairs (or larger groups) with shared input weights, each group's columns of $W_{l+1} Z$ must sum to zero, but individual columns may be nonzero.
\end{itemize}

\paragraph{Relation to duplication.} The 0-type expansion is complementary to neuron duplication. Duplication copies existing neurons and redistributes their output weights (the new neurons are functionally redundant copies of existing ones). The 0-type expansion adds genuinely new neurons whose total output cancels. Together, these two operations generate the full set of overparameterization symmetries for arbitrary activations described in \cref{sec:general}.

\begin{tcolorbox}[title=Example: Adding 0-type Neurons, colback=yellow!5]
For a network with $n = 2$ hidden neurons and $d_{\mathrm{out}} = 2$, adding two 0-type neurons sharing their input weight (rows 3 and 4 of $M_0$ are equal):
\[
M_0 = \begin{bmatrix} 1 & 0 \\ 0 & 1 \\ a_1 & a_2 \\ a_1 & a_2 \end{bmatrix}, \qquad
M_0^+ = \begin{bmatrix} 1 & 0 & z_1 & -z_1 \\ 0 & 1 & z_2 & -z_2 \end{bmatrix}.
\]
Both new neurons share the pre-activation $(a_1 W_l[1,:] + a_2 W_l[2,:]) x$, and their output contributions $W_{l+1} z\, \sigma(\cdot)$ and $-W_{l+1} z\, \sigma(\cdot)$ cancel for any $z = (z_1, z_2)^\top$. The parameters $a_1, a_2, z_1, z_2$ are all free.
\end{tcolorbox}
% \paragraph{Effect on optimization.} Since the new neurons have zero output weights, backpropagation gives zero gradients for their input weights, and the Hessian gains additional zero eigenvalues corresponding to these fully flat directions. The existing eigenvalues are unchanged: the $B$ matrix for this transformation simply pads with new dimensions rather than rescaling existing ones. This means 0-type neurons contribute to Mechanism 2 (volume growth, by expanding the solution manifold) but not to Mechanism 1 (preconditioning).

\section{Per-Weight Splitting and the Inverse Transformation}
\label{app:per_weight}

\subsection{Per-Weight Coefficient Matrices}

The formulation in \citet{simsek2021geometry} requires scaling the entire outgoing weight vector together. However, there exists a strictly richer symmetry: the splitting coefficients need not be shared across the entire outgoing weight vector.

Let $W_{l+1} \in \R^{d \times n}$ have columns $\va_1, \va_2, \ldots, \va_n$. If column $i$ is duplicated $m_i$ times, we can choose a corresponding coefficient matrix $C_i \in \R^{d \times m_i}$ with normalized rows: $\sum_{j=1}^{m_i} (C_i)_{kj} = 1$ for all rows $k$.

The transformed output weights are:
\begin{equation}
W_{l+1} \to \begin{bmatrix}
    C_1 \odot \va_1 & C_2 \odot \va_2  \, \dots \, C_n \odot \va_n
\end{bmatrix},
\end{equation}
where $\odot$ denotes element-wise multiplication and broadcasting applies column-wise (equivalently, the $i$-th block is given by $\mathrm{diag}(\va_i)C_i$).

\begin{tcolorbox}[title=Example: Per-Weight Splitting Coefficients, colback=green!5]
For $M = \begin{bsmallmatrix} 1 & 0 \\ 1 & 0 \\ 0 & 1 \end{bsmallmatrix}$, we can transform a $2 \times 2$ matrix:
\[
W_{l+1} = \begin{bmatrix} a_1 & a_2 \\ a_3 & a_4 \end{bmatrix}
\to
\begin{bmatrix} \mu_1 a_1 & (1-\mu_1) a_1 & a_2 \\ \mu_2 a_3 & (1-\mu_2) a_3 & a_4 \end{bmatrix}
\]
for constants $\mu_1, \mu_2 \in \R$. Crucially, $\mu_1$ and $\mu_2$ can be chosen independently, giving per-weight control rather than per-neuron control.
\end{tcolorbox}

\subsection{Row-wise Hessian Scaling}
\label{app:rowwise}

The per-weight splitting above forces us to scale rows/columns in the Hessian as neuron-wise blocks. Here we discuss the gradient behavior when splits use individual weight-wise coefficients $c_{jk}$.

\begin{theorem}
Using the notation from \cref{sec:gradients}, if our splits use individual coefficients per weight $c_{jk}$, we get:
\begin{align}
\frac{\partial L}{\partial w'^{(l)}_{ij}} &= h^{(l-1)}_i (\sigma^l)' \sum_k c_{jk} w_{jk}^l \delta^{(l+1)}_k \\
&= \frac{\sum_k c_{jk} w_{jk}^l \delta^{(l+1)}_k}{\sum_k w_{jk}^l \delta^{(l+1)}_k} \frac{\partial L}{\partial w^{(l)}_{ij}} \\
&:= d_j \frac{\partial L}{\partial w^{(l)}_{ij}}.
\end{align}
\end{theorem}

The coefficient $d_j$ only depends on index $j$, so we cannot control the gradient of individual weights independently.
% However, composing post-activation splitting and rescale operations may enable finer control.

\subsection{Inverse Transformation}
\label{app:inverse}

Given an expanded network with duplication matrix $M \in \R^{r \times n}$ and contraction $M^+$, the inverse transformation recovers the original width-$n$ parameterization. For post-activation splitting, where neuron $i$ was duplicated $m_i$ times, the inverse groups each set of $m_i$ copies back into a single neuron: the input weights are averaged across copies (recovering the original, since all copies are identical by construction), and the output weights are summed.

Concretely, the inverse uses the even-split pseudoinverse $M^+_{\text{even}} \in \R^{n \times r}$ for the input side and the original expansion matrix $M \in \R^{r \times n}$ for the output side:
\begin{equation}
W_l' \to M^+_{\text{even}}\, W_l', \quad W_{l+1}' \to W_{l+1}'\, M,
\end{equation}
where $(M^+_{\text{even}})_{i,j} = 1/m_i$ if column $j$ is a copy of neuron $i$ and $0$ otherwise. The input averaging recovers the shared weights since $M^+_{\text{even}} M = I_n$. Right-multiplying by $M$ naturally sums the split output weights: because the expanded weights are $W_{l+1}' = W_{l+1} M^+$, the inverse operation yields $W_{l+1}' M = W_{l+1} M^+ M = W_{l+1} I_n = W_{l+1}$.

For pre-activation splitting (ReLU), the roles reverse: output weights are averaged and input weights are summed to undo the splitting coefficients $D_{m_i}$.

Note that the inverse is only well-defined on the image of the forward transformation i.e., on networks whose weights exhibit the tied structure produced by expansion. This is why the symmetries form a groupoid rather than a group (\cref{sec:groupoid}).

\begin{tcolorbox}[title=Example: Inverse of Post-Activation Splitting, colback=red!5]
For $M = \begin{bsmallmatrix} 1 & 0 \\ 1 & 0 \\ 0 & 1 \end{bsmallmatrix}$ (duplicating neuron 1), suppose the expanded network has:
\[
W_l' = \begin{bmatrix} \vw_1 \\ \vw_1 \\ \vw_2 \end{bmatrix}, \quad
W_{l+1}' = \begin{bmatrix} \mu_1 a_1 & (1-\mu_1) a_1 & a_2 \\ \mu_2 a_3 & (1-\mu_2) a_3 & a_4 \end{bmatrix}.
\]
The inverse uses $M^+_{\text{even}} = \begin{bsmallmatrix} 1/2 & 1/2 & 0 \\ 0 & 0 & 1 \end{bsmallmatrix}$ and the original matrix $M$:
\[
M^+_{\text{even}}\, W_l' = \begin{bmatrix} \vw_1 \\ \vw_2 \end{bmatrix},
\]
\[
W_{l+1}'\, M = \begin{bmatrix} \mu_1 a_1 + (1-\mu_1) a_1 & a_2 \\ \mu_2 a_3 + (1-\mu_2) a_3 & a_4 \end{bmatrix} = \begin{bmatrix} a_1 & a_2 \\ a_3 & a_4 \end{bmatrix},
\]
recovering the original parameterization.
\end{tcolorbox}

\section{Transformation of Weight Space, Gradients and Hessian under Splitting}
\label{app:vectorize}

For an $L$-layer neural network, assume we insert a transformation pair $M_l^+ M_l = I$ between every adjacent layer $l$ and $l+1$. The forward parameters transform as follows:
\begin{itemize}
    \item First layer: $W_1 \to M_1 W_1$
    \item Middle layers ($1 < l < L$): $W_l \to M_l W_l M_{l-1}^+$
    \item Final layer: $W_L \to W_L M_{L-1}^+$
\end{itemize}

By vectorizing the weights using the identity $\mathrm{vec}(ABC) = (C^T \otimes A)\mathrm{vec}(B)$, the parameter transformation matrix $\mathbf{P}$ across all layers is the block-diagonal matrix:
\begin{equation}
    \mathbf{P} = \mathrm{diag}\left(I \otimes M_1, \dots, (M_{l-1}^+)^T \otimes M_l, \dots, (M_{L-1}^+)^T \otimes I \right)
\end{equation}

Because parameters are contravariant while gradients are covariant, the gradient vector $\nabla L_\theta$ transforms via the transpose of the pseudo-inverse of the parameter transformation: $\nabla L_{\theta'} = (\mathbf{P}^+)^T \nabla L_\theta$.

Letting $B^T = (\mathbf{P}^+)^T$, the correct transformation for the expanded gradient vector across the entire network is $\nabla L_{\theta'} = B^T \nabla L_\theta$, where $B^T$ is the block-diagonal matrix:
\begin{equation}
    B^T = \mathrm{diag}\left(I \otimes (M_1^+)^T, \dots, M_{l-1} \otimes (M_l^+)^T, \dots, M_{L-1} \otimes I \right)
\end{equation}

By the multivariate chain rule, the Jacobian of this transformation is simply the constant matrix $B$:
\begin{equation}
    J = \frac{\partial \theta}{\partial \theta'} = B
\end{equation}

The first derivative (the gradient) transforms via the transposed Jacobian:
\begin{equation}
    \nabla_{\theta'} L = J^T \nabla_\theta L = B^T \nabla_\theta L
\end{equation}

To find the transformed Hessian $H_{\theta'} = \nabla^2_{\theta'} L$, we differentiate the transformed gradient with respect to the new parameters $\theta'$:
\begin{equation}
    H_{\theta'} = \frac{\partial}{\partial \theta'} \left( B^T \nabla_\theta L \right)
\end{equation}

Because $B$ is a matrix of constants, its derivative with respect to $\theta'$ is zero. Thus, the product rule yields no higher-order derivative tensor terms. Applying the chain rule, we extract the original Hessian $H_\theta = \frac{\partial}{\partial \theta}(\nabla_\theta L)$ and multiply by the Jacobian $J = B$:
\begin{align}
    H_{\theta'} &= B^T \left( \frac{\partial (\nabla_\theta L)}{\partial \theta} \right) \frac{\partial \theta}{\partial \theta'} \nonumber \\
                &= B^T H_\theta B
\end{align}

The gradient identity $\nabla_{\theta'} L = B^\top \nabla_\theta L$ holds on the symmetry orbit. The corresponding Hessian relation picks up an extra term proportional to $\nabla L$ off critical points, so $H_{\theta'} = B^\top H_\theta B$ holds exactly at the zero-loss minima we study. More generally, the relation holds exactly whenever the Gauss-Newton approximation to the Hessian is exact.

In overparameterized networks, almost all local minima are globally optimal \citep{nguyen2017loss}. The residual $e \nabla^2 f$ vanishes at these zero-loss minima, so $H_{\theta'} = B^\top H_\theta B$ holds exactly throughout the minimum landscape.

\paragraph{Block form for even splits.}
Partitioning the weights into $O$ (other), $S$ (split), and $D$ (duplicated) with $|S|=|D|$, an order-$m$ even split transforms the Hessian as
\begin{equation}
H = \begin{bmatrix}
H_{OO} & H_{OS} & H_{OD} \\
H_{OS}^\top & H_{SS} & H_{SD} \\
H_{OD}^\top & H_{SD}^\top & H_{DD}
\end{bmatrix}
\;\to\;
H' = \begin{bmatrix}
H_{OO} & H'_{OS} & \tfrac{1}{m} H'_{OD} \\
H_{OS}'^\top & H'_{SS} & \tfrac{1}{m} H'_{SD} \\
\tfrac{1}{m} H_{OD}'^\top & \tfrac{1}{m} H_{SD}'^\top & \tfrac{1}{m^2} H'_{DD}
\end{bmatrix},
\end{equation}
where $H'_{XX} = H_{XX} \otimes \mathbf{1}_{m \times m}$. Every duplicated row/column appears $m$ times in its block; the $1/m$ and $1/m^2$ factors come from the $\alpha_i = 1/m$ row scalings of $B$ for duplicated weights.

\section{The Overparameterization Groupoid}
\label{app:groupoid}

Let $\Gamma^{(L,N)}$ be the set of functions exactly representable by a neural network with $L$ hidden layers and $N$ neurons per layer. Notice that $\Gamma^{(L,N)} \subset \Gamma^{(L+p,N+q)}$ for any $p,q \in \mathbb{N}$. Let $\Theta^{L}_N$ be the corresponding parameter space of weights for such a network. For brevity, we denote the single-layer parameter space as $\Theta_N = \Theta^{1}_N$, and the union of parameter spaces with width at least $a$ as $\Theta^{L}_{\ge a} = \bigcup_{N=a}^\infty \Theta^{L}_N$.

To formalize overparameterization (i.e., post-activation splitting) that preserves the network's function, we define operations mapping between these parameter spaces. For $b \ge a$, let $T_{a \rightarrow b}$ be the set of single-layer expansion operations:
\begin{equation}
T_{a \rightarrow b} = \{(M_{a\rightarrow b}, C_1, C_2 \dots C_a) \mid M_{a\rightarrow b} \text{ and } C_i \text{ are valid expansions from } a \text{ to } b\},
\end{equation}
where $M_{a\rightarrow b}$ and $C_i$ are the expansion and coefficient matrices from \cref{sec:characterizing}. Any transformation $t \in T_{a \rightarrow b}$ maps $\Theta_a \rightarrow \Theta_b$.

Conversely, we define a set of inverse operations, $T_{b \rightarrow a}$, which are strictly defined on the image of the expansions, $\im(T_{a \rightarrow b}) \subset \Theta_b$. Specifically, any $t \in T_{b \rightarrow a}$ takes $b-a$ duplicated pre-activation weights and collapses them into a single weight, while summing the corresponding output weights. These duplicated weights are guaranteed to exist for any $\theta \in \im(T_{a \rightarrow b})$ (see Appendix~\ref{app:inverse}).

For an entire $L$-layer network, the transformations are applied layer-wise. We denote the set of full-network expansions as the Cartesian product $T^L_{a\rightarrow b} = (T_{a\rightarrow b})^L$, with inverses defined similarly.

We now construct an algebraic structure to capture all valid sequences of these transformations. We must use a \textbf{groupoid} rather than a group for two reasons:
\begin{enumerate}
    \item \textbf{Partial composition:} A transformation mapping $\tau_{a\rightarrow b}$ can only be composed with a subsequent transformation $\tau_{b\rightarrow c}$. The binary operation is only defined when input and output widths match.
    \item \textbf{Conditional invertibility:} We can only invert an expansion (apply a contraction) if the weights are perfectly tied as a result of a previous expansion.
\end{enumerate}

Let $n_f$ be the minimum width such that the network's underlying function $f_\theta \in \Gamma^{(L,n_f)}$. We define the overparameterization groupoid $\Onl$ over the parameter space $\Theta^L_{\ge n_f}$. $\Onl$ is formally generated by the collection of all valid layer-wise expansions, their inverses, and the identity operations $I_N$:
\begin{equation}
\Onl = \langle T^L_{a \rightarrow b} \cup T^L_{b \rightarrow a} \cup \{I_N\} \mid \forall a, b, N \ge n_f \rangle.
\end{equation}
The partial binary operation $*$ is standard function composition, which is inherently associative. Because $\Onl$ contains all identities, inverses on valid images, and associative partial composition, it satisfies the groupoid axioms. A key property of $\Onl$ is that it describes the exact orbit of parameters that leave the underlying function invariant; for any $\theta \in \Theta^L_{\ge n_f}$ and valid $g \in \Onl$, $f_\theta = f_{g \cdot \theta}$.

We remark that even some weight-space symmetries at a fixed width can be viewed as ``local'' or groupoid symmetries; \citet{grigsby2023hidden} provide a detailed analysis of such hidden symmetries in ReLU networks.

\section{Optimization Intuition via Rescale Symmetry}
\label{sec:rescale}

Before analyzing the overparameterization symmetries from the main text, we build intuition with the simpler rescale symmetry, which is present even in exactly-parameterized networks. The core argument carries over to the overparameterization case. It is well documented that rescale symmetries impact the curvature of minima and can be exploited to improve optimization \citep{huang2017projection, salimans2016weight, neyshabur2015path, mishkin2025levelset}.

Consider a two-layer ReLU network $f(x; W, U) = U \cdot \text{ReLU}(Wx)$. Since ReLU is equivariant to positive rescaling, the transformation $g_\lambda(W, U) = (\lambda W, \lambda^{-1} U)$ for $\lambda \in \R^+$ preserves the function. More generally, each hidden neuron can be rescaled independently.

\paragraph{Gradients under rescaling.} At a rescaled point $(\lambda W, \lambda^{-1} U)$, the gradients transform as:
\begin{equation}
\frac{\partial L}{\partial W'} = \lambda^{-1} \frac{\partial L}{\partial W}, \qquad \frac{\partial L}{\partial U'} = \lambda \frac{\partial L}{\partial U}.
\end{equation}
This follows from the chain rule: $W' = \lambda W$ so $\partial L / \partial W'_{ij} = (\partial W_{ij} / \partial W'_{ij}) \cdot \partial L / \partial W_{ij} = \lambda^{-1} \partial L / \partial W_{ij}$, and similarly for $U$. Rescaling thus redistributes gradient magnitudes between the input and output weights. When $\lambda > 1$, input weight gradients shrink while output weight gradients grow, and vice versa. This is analogous to the gradient scaling under splitting (\cref{thm:gradient}), where the splitting coefficients play the role of $\lambda$.

\paragraph{Hessian under rescaling.} If $(W^*, U^*)$ is a global minimum, rescaling gives a class of equivalent minima $[(W^*, U^*)]_\Tr$. The minimum $(W^*, U^*)$ can be arbitrarily poorly conditioned, but the symmetry orbit contains solutions with different conditioning. The Hessian at $(W^*, U^*)$ has the block structure:
\begin{equation}
H = \begin{bmatrix} H_{WW} & H_{WU} \\ H_{WU}^\top & H_{UU} \end{bmatrix}.
\end{equation}

At the rescaled point $(\lambda W^*, \lambda^{-1} U^*)$, the Hessian becomes:
\begin{equation}
H(\lambda) = \begin{bmatrix} \lambda^{-2} H_{WW} & H_{WU} \\ H_{WU}^\top & \lambda^{2} H_{UU} \end{bmatrix}.
\end{equation}

This is diagonal preconditioning with $K = \diag(\lambda^{-2} I, \lambda^2 I)$, matching the structure from \cref{sec:preconditioning}. Choosing $\lambda$ lets us trade off curvature between the $W$ and $U$ blocks.

\paragraph{Limitations.} While $\min_\lambda \kappa(H(\lambda)) \leq \kappa(H)$ always holds, rescaling cannot improve conditioning universally. To see this, consider the Rayleigh quotient for $x = (u, v)$, which is a weighted average of eigenvalues:

\begin{equation}
R(H(\lambda), x) = \frac{\lambda^{-2} u^\top H_{WW} u + 2 u^\top H_{WU} v + \lambda^{2} v^\top H_{UU} v}{\|u\|^2 + \|v\|^2}.
\end{equation}
The off-diagonal term $u^\top H_{WU} v$ is unaffected by $\lambda$, so if the extreme eigenvalues are dominated by the cross-block coupling, rescaling has little effect. Similarly, if both extreme eigenvalues come from a single block (say $H_{WW}$), rescaling scales them together and cannot change their ratio. Conditioning improves when the largest and smallest eigenvalues are split across the $H_{WW}$ and $H_{UU}$ blocks, which corresponds to the eigenvectors being approximately axis-aligned in the $(W, U)$ decomposition. This is depicted in \cref{fig:rescale}.

\section{Which Activations Allow Pre-Activation Splitting?}
\label{app:activations}

Understanding general pre-activation splitting requires solving the functional equation representing a single split:
\begin{equation}
\forall x \in \mathbb{R}, \quad \sigma(\alpha x) + \sigma(\beta x) = \sigma(x),
\end{equation}
for some expansion constants $\alpha, \beta$. A direct observation is that setting $x=0$ yields $2\sigma(0) = \sigma(0)$, implying $\sigma(0) = 0$. This immediately rules out activations with non-zero origins, such as the standard sigmoid.

To solve this systematically, we first restrict our attention to the domain $x > 0$ and assume $\alpha, \beta > 0$. We can map this to a continuous difference equation by substituting $x = e^t$ and defining $y(t) = \sigma(e^t)$. Letting $\alpha = e^a$ and $\beta = e^b$, the equation becomes:
\begin{equation}
y(t+a) + y(t+b) = y(t).
\end{equation}

Applying the Laplace transform (assuming mild growth conditions on $y$) yields $(e^{as} + e^{bs} - 1)Y(s) = 0$. For non-trivial solutions ($Y(s) \neq 0$), $s$ must be a root of the characteristic equation, meaning $s$ lies in the set:
\begin{equation}
S_{a,b} = \{s \in \mathbb{C} \mid e^{as} + e^{bs} = 1\}.
\end{equation}

If the ratio $a/b$ is rational (i.e., $a/b = p/q$ with $\gcd(p,q)=1$), the substitution $u = e^{s/q}$ reduces this to a polynomial equation $u^p + u^q = 1$, which has exactly $\max(p,q)$ roots for $u$ by the Fundamental Theorem of Algebra. Taking the inverse Laplace transform for a given principal root yields base solutions of the form $y(t) = C e^{st}$, which translates back to $\sigma(x) = C x^s$, subject to the algebraic constraint $\alpha^s + \beta^s = 1$. When $s$ is strictly real, these correspond to standard monomial functions.

However, we must also consider complex roots. When $s = \lambda + iw$, the solutions exhibit oscillatory behavior. For a purely imaginary root ($s = iw$, so $\lambda=0$), the real-valued solutions take the form $\sigma(x) = \cos(w \ln x)$. Returning to the difference equation, this requires $\cos(wt + aw) + \cos(wt + bw) = \cos(wt)$. By angle addition, this holds when the phases form an equilateral triangle in the complex plane, yielding $(aw, bw) = (\pm \frac{\pi}{3} + 2\pi m, \mp \frac{\pi}{3} + 2\pi n)$ for $m,n \in \mathbb{Z}$. Because the limit $\lim_{x \to 0^+} \cos(w \ln x)$ is undefined, these purely oscillatory solutions are pathological and impractical for neural networks.

A careful treatment of signs maps our real roots ($s=1, 2, \dots$) directly to practical piecewise-polynomial activations. This framework naturally recovers both the standard ReLU ($s=1$) and ReLU$^2$ ($s=2$), which has recently emerged as a performant choice in sparse LLMs \cite{zhang2024relu2winsdiscoveringefficient}.

\section{Axis-Algined Eigenvectors and Transformer Rotation Symmetries}
\label{app:transformers}

We empirically check the axis alignment of the Hessian eigenvectors of pretrained
teacher networks across architectures.

In \ref{fig:axis_alignment}, we compute the top-1000 Hessian eigenpairs via
Lanczos at the trained parameters, and for each eigenvector $v$ measure the
participation ratio $\mathrm{PR}(v) = 1 / \sum_i v_i^4$ on the parameter basis. The
fourth power simply makes $\mathrm{PR}$ an interpretable count: if $v$ has equal mass
on $k$ coordinates and zero on the rest, then $\sum_i v_i^4 = 1/k$, so
$\mathrm{PR}(v) = k$. A lower power would not give this property: $\sum_i v_i^2 = 1$
for any unit $v$, so it cannot distinguish a one-hot eigenvector from a fully spread
one. We report $\mathrm{PR}(v) / n_{\mathrm{params}}$, the fraction of parameters the
eigenvector effectively occupies; a uniformly random unit vector in
$\mathbb{R}^{n_{\mathrm{params}}}$ has expected fraction $\approx 1/3$.

Across the three architectures we tested (a CNN trained on CIFAR-10, an
absolute-position character Transformer pretrained on Alice in Wonderland, and the same
Transformer with RoPE positional encoding) the Hessian eigenvectors lie below this
random baseline (Figure~\ref{fig:axis_alignment}), indicating a degree of axis
alignment relative to random directions. The localization is strongest at the largest
eigenvalues and decays towards the random baseline as one moves down the spectrum.

\begin{figure}[h]
\centering
\includegraphics[width=\linewidth]{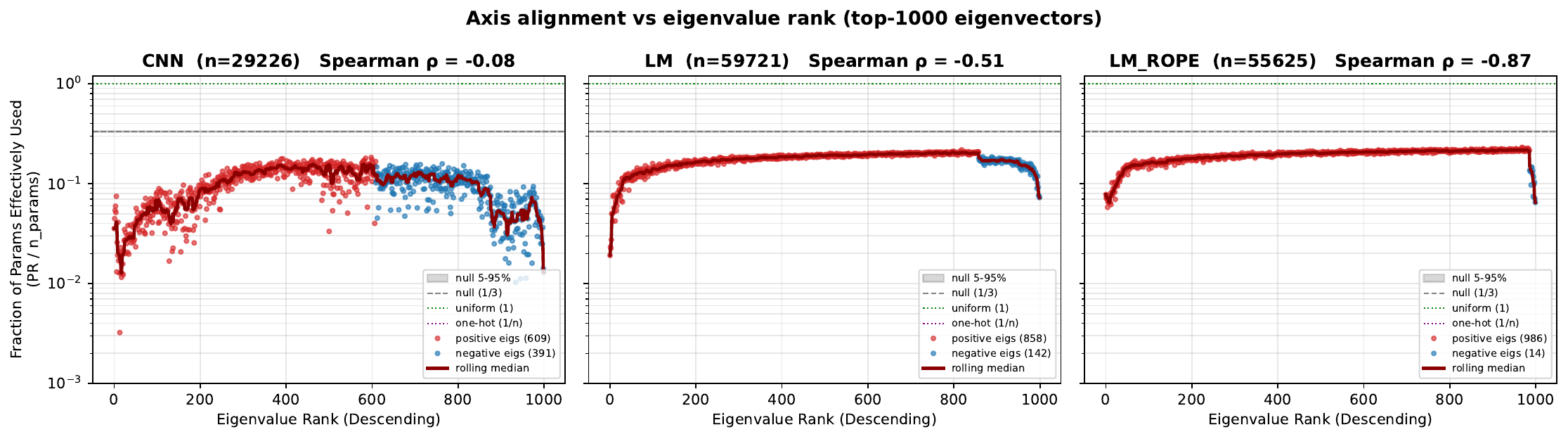}
\caption{Fraction of parameters effectively used by each of the top-1000 Hessian
eigenvectors of three pretrained teachers, against a random-unit-vector null
($\approx 1/3$, gray band). Red points correspond to positive eigenvalues, blue to
negative; the dark red line is a rolling median.}
\label{fig:axis_alignment}
\end{figure}

To see where this localization sits in parameter space, we decompose each top
eigenvector $v$ into per-tensor energy fractions
$\sum_{i \in T} v_i^2$ for each parameter tensor $T$, and divide by the
parameter-count fraction $|T| / n_{\mathrm{params}}$ that a uniform random direction
would assign to $T$. The resulting quantity (one per (eigenvector, tensor) pair) takes
value 1 for a uniform direction and grows when the eigenvector concentrates more
energy on a tensor than its size alone would predict. Figure~\ref{fig:layer_decomp}
shows this decomposition for the top-20 eigenvectors of the RoPE Transformer teacher.
The brightest columns correspond to the per-block additive biases on the residual
stream (the attention output biases \texttt{attn.out.b} and the FFN output biases
\texttt{ff.2.b} of every layer), each of which contains only 32 of the
$\sim$56k model parameters but absorbs a disproportionate share of each top
eigenvector's energy.

\begin{figure}[h]
\centering
\includegraphics[width=\linewidth]{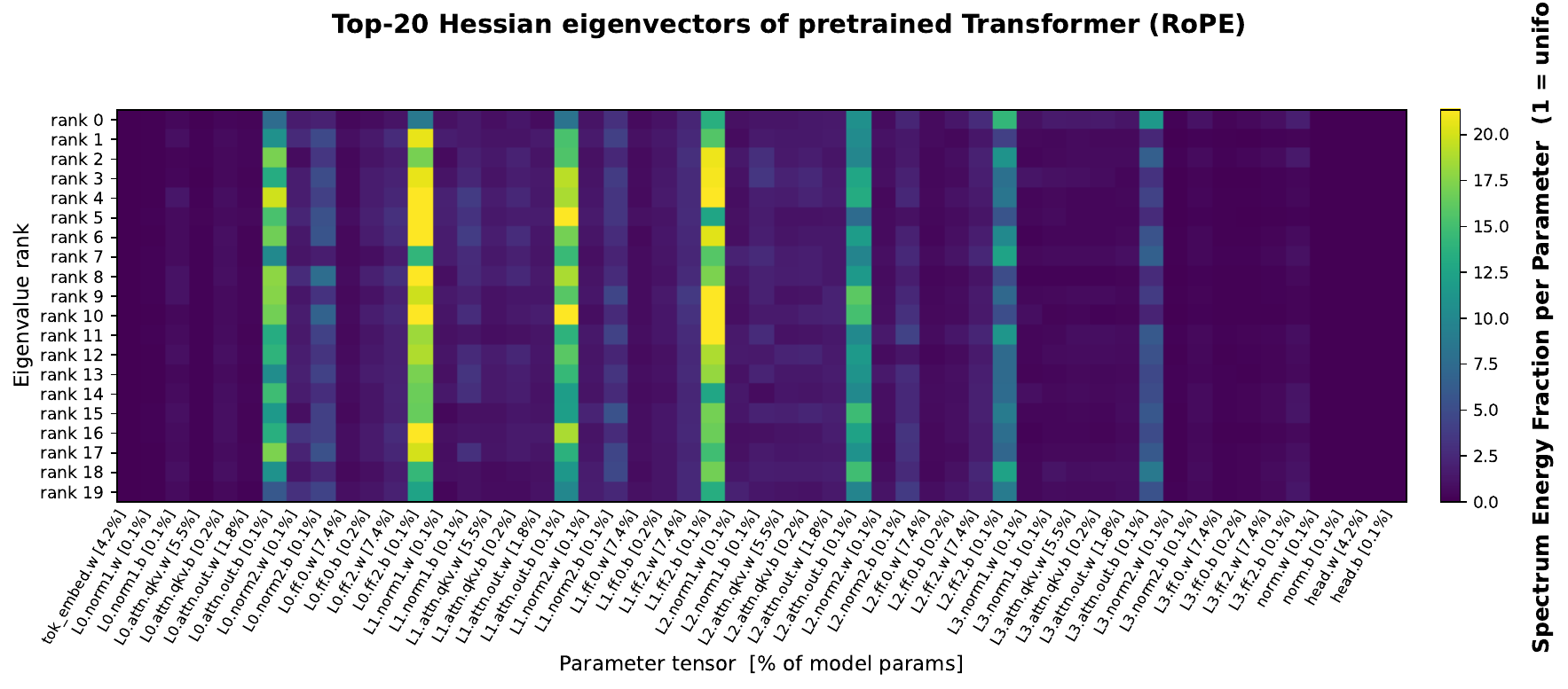}
\caption{Spectrum energy fraction per parameter for the top-20 Hessian eigenvectors
of the RoPE Transformer teacher, broken down by parameter tensor. Each column is a
parameter tensor (annotated with its share of total parameters); each row is an
eigenvector ranked by descending eigenvalue. Brighter cells indicate more
eigenvector energy per parameter than a uniform random direction would assign.}
\label{fig:layer_decomp}
\end{figure}

Transformers have an additional rotation symmetry in their MLP weights, partially described in \citet{zhang2025permutation} for networks without skip connections. In modern pre-norm Transformers with skip connections, the symmetry is more restricted. In general, the symmetry performs a rotation in the MLP weight matrices $W \to R_{\text{left}} W R_{\text{right}}^\top$ for some special orthogonal matrices $R_{\text{left}}, R_{\text{right}} \in O(n)$, acting jointly across adjacent blocks to preserve the residual stream.

To see how this rotates weight space, we apply the Kronecker product identity:
\begin{equation}
\text{vec}(R_{\text{left}} W R_{\text{right}}^\top) = (R_{\text{left}} \otimes R_{\text{right}}) \text{vec}(W).
\end{equation}

Intuitively, this rotates in unison the column weights by $R_{\text{left}}$ and the row weights by $R_{\text{right}}$.

This rotation symmetry can aid optimization algorithms with built-in diagonal preconditioners, such as Adam. Recent work has noted the underperformance of Adam when rotating the preconditioner in Transformers but not ResNets (which lack this rotation symmetry), suggesting the minima Adam converges to typically have axis-aligned eigenvectors \citep{zhang2025understanding}.

\section{Notes on Plasticity Loss}
\label{app:plasticity}

A direct application of our framework concerns plasticity loss in continual learning. If optimization converges to a simple generalizing function, the solution lies on a low functional width column of the overparameterization groupoid (\cref{fig:groupoid}). Such solutions are typically realized through dormant capacity such as duplicated or dead neurons, configurations that are fixed points of subgroups of the network's symmetry group. It is a well-known result that gradient descent preserves these fixed-point subspaces under any loss: cloned neurons stay cloned, dead neurons stay dead, throughout training. Once the objective changes, the optimizer cannot recruit this dormant capacity to fit the new task.

\citet{ziyin2023symmetry} previously analyzed this trapping mechanism for permutation, rescaling, and rotation symmetries under weight decay or noise, and explicitly linked it to plasticity loss. A more detailed and concurrent dynamical systems treatment appears in \citet{joudaki2025barriers}.

\section{Local Volume of Minima}
\label{app:ellipsoid_volume}

We compute the volume of the local $\epsilon$-sublevel set used in \cref{sec:volume_properties}. Taylor-expanding the loss near a minimum $\theta^*$, the local ellipsoidal region $\mathcal{E}_\epsilon(\theta^*)$ where the loss increases by at most $\epsilon$ satisfies $\tfrac{1}{2}(\theta - \theta^*)^\top H (\theta - \theta^*) \leq \epsilon$.

Assuming the parameter space has dimension $n$, the affine change of variables $u = \frac{1}{\sqrt{2\epsilon}} H^{1/2} (\theta - \theta^*)$ maps the region to the unit $n$-ball $u^\top u \leq 1$, with Jacobian determinant $(2\epsilon)^{n/2} (\det(H))^{-1/2}$. Integrating over the unit ball gives
\begin{equation}
  \Vol(\mathcal{E}_\epsilon(\theta^*)) = \frac{(2\epsilon)^{n/2}\,\Vol(B_{n})}{\sqrt{\det(H)}},
  \label{eq:ellipsoid_general}
\end{equation}
where $\Vol(B_{n}) = \pi^{n/2}/\Gamma(n/2 + 1)$. When $H$ is singular (as for $H'$ obtained by splitting), $\det(H)$ is replaced by the pseudo-determinant $\pdet(H)$, giving the volume modulo flat symmetry directions used in \cref{thm:pdet}.

\section{Derivation of the Volume Ratio}
\label{app:volume_ratio}

We derive the probability ratio when expanding from $n$ to $n+1$ neurons (i.e.\ $k=1$). Throughout we fix $d_{\mathrm{in}} = d_{\mathrm{out}} = 1$, so the base network has $2n$ parameters and the expanded network has $2(n+1)$ parameters. Both networks use the same constant $c > 0$, with the $m$-neuron network initialized uniformly on $[-b_m, b_m]^{2m}$ where $b_m = 1/\sqrt{m}$.

\paragraph{Domain assumption.}
We restrict attention to base minima $\theta^*$ whose weights all satisfy $|w_j^{\mathrm{in}}|, |w_j^{\mathrm{out}}| \leq b_{n+1}$ for every $j$. Under this condition, for any $\alpha \in [0,1]$ the split minimum $\theta^*_{\alpha,j}$ lies entirely inside $[-b_{n+1}, b_{n+1}]^{2(n+1)}$: the shared incoming weights are bounded by $b_{n+1}$ by assumption, and $|\alpha w_j^{\mathrm{out}}|, |(1-\alpha)w_j^{\mathrm{out}}| \leq |w_j^{\mathrm{out}}| \leq b_{n+1}$. The entire interval $\alpha \in [0,1]$ is therefore admissible.

\subsection{Base network: $P_n$}

\paragraph{Volume of the $\epsilon$-sublevel ellipsoid.}
Fix a global minimum $\theta^{*}$. The quadratic approximation of the loss defines an ellipsoid $\{\theta : \frac{1}{2}(\theta - \theta^*)^\top H(\theta-\theta^*) \leq \epsilon\}$ in $\R^{2n}$. Assuming $H$ is full rank (no flat directions for the base network), the ellipsoid has volume
\begin{equation}
  \Vol(\mathcal{E}_\epsilon(\theta^*))
  \;=\;
  \frac{(2\epsilon)^{n}\,\Vol(B_{2n})}{\sqrt{\det(H)}},
  \label{eq:ellipsoid_base}
\end{equation}
where $\Vol(B_{2n}) = \pi^{n}/\Gamma(n+1)$.

\paragraph{Summing over permutation copies.}
Permuting the $n$ neurons gives $n!$ equivalent minima. Assuming the $\epsilon$-sublevel sets around distinct permutation copies are disjoint (which holds when $\epsilon$ is small relative to the separation between minima), and that each ellipsoid lies inside $[-b_n, b_n]^{2n}$ (which holds when $\epsilon$ is small relative to the distance from $\theta^*$ to the boundary), the probability of initializing within loss $\epsilon$ of any global minimum of this type is
\begin{equation}
  \boxed{P_n
  \;=\;
  \frac{n!\;(2\epsilon)^{n}\,\Vol(B_{2n})}
       {(2b_n)^{2n}\,\sqrt{\det(H)}}.}
  \label{eq:Pn}
\end{equation}

\subsection{Expanded network: $P_{n+1}$}

\paragraph{Structure of minima.}
For the $k=1$ expansion, a global minimum of the expanded network arises by splitting exactly one neuron $j \in \{1,\ldots,n\}$ into two copies with outgoing weights $\alpha w_j^{\mathrm{out}}$ and $(1-\alpha) w_j^{\mathrm{out}}$ for some $\alpha \in [0,1]$, while both copies share the incoming weight $w_j^{\mathrm{in}}$. The corresponding parameter vector is
\[
  \theta^*_{\alpha,j}
  \;=\;
  \bigl(\ldots,\;w_j^{\mathrm{in}},\,\alpha w_j^{\mathrm{out}},\;
               w_j^{\mathrm{in}},\,(1-\alpha) w_j^{\mathrm{out}},\;\ldots
  \bigr)
  \;\in\;\R^{2(n+1)}.
\]

\paragraph{Pseudo-determinant of the expanded Hessian.}
By \cref{thm:pdet} applied to the single split, the expanded Hessian $H'$ has $2n$ nonzero eigenvalues and $2$ zero eigenvalues. Its pseudo-determinant satisfies
\begin{equation}
  \pdet(H')
  \;=\;
  \det(H)\;\cdot\;2\bigl(\alpha^2 + (1-\alpha)^2\bigr).
  \label{eq:pdet_k1}
\end{equation}

\paragraph{Flat directions.}
The $2$ zero eigenvalues correspond to two distinct directions.

\begin{itemize}
  \item \emph{Globally flat direction}: shifting $\alpha \mapsto \alpha + t$ (redistributing the outgoing weight) leaves the network function unchanged. We integrate over this direction explicitly below.

  \item \emph{Locally flat direction}: the null direction of the Hessian block for the two copies' incoming weights is $v_\alpha = (1-\alpha, -\alpha)$, since linearizing the network output in incoming-weight perturbations $(\delta w_\mathrm{in}^{(1)}, \delta w_\mathrm{in}^{(2)})$ gives a change proportional to $\alpha\,\delta w_\mathrm{in}^{(1)} + (1-\alpha)\,\delta w_\mathrm{in}^{(2)}$. Consider a perturbation size $s$ along $v_\alpha$ giving a set of parameters $\theta_s$. Along this direction the first order change in $f_{\theta_s}$ is $0$, and since the loss is quadratic in $f_{\theta_s}$ it grows at most quartically (as $s^4$), so the sublevel set $\{L \leq \epsilon\}$ has width $O(\epsilon^{1/4})$, contributing a factor $(\epsilon/C)^{1/4}$ for a constant $C > 0$ taken as a uniform upper bound on the fourth-order coefficient over $\alpha \in [0, 1]$ and admissible teacher minima.
\end{itemize}

\paragraph{Volume in the nonflat subspace.}
In the $2n$-dimensional subspace orthogonal to both flat directions:
\begin{equation}
  \Vol\!\bigl(\mathcal{E}_\epsilon(\theta^*_{\alpha,j})^{\perp}\bigr)
  \;=\;
  \frac{(2\epsilon)^{n}\,\Vol(B_{2n})}{\sqrt{\pdet(H')}}
  \;=\;
  \frac{(2\epsilon)^{n}\,\Vol(B_{2n})}
       {\sqrt{\det(H)}\cdot\sqrt{2(\alpha^2+(1-\alpha)^2)}}.
  \label{eq:ellipsoid_expanded}
\end{equation}

\paragraph{Integrating over the globally flat direction.}
The orbit is the line segment $\theta^*_\alpha = (\ldots, w_j^\mathrm{in}, \alpha w_j^\mathrm{out}, w_j^\mathrm{in}, (1-\alpha) w_j^\mathrm{out}, \ldots)$ with tangent vector $\partial_\alpha \theta^*_\alpha$ of norm $\sqrt{2}\,|w_j^\mathrm{out}|$, so the arc-length element is $d\ell = \sqrt{2}\,|w_j^\mathrm{out}|\, d\alpha$. By the domain assumption, the full interval $\alpha \in [0,1]$ is admissible. Integrating the per-$\alpha$ cross-section with respect to arc length,
\begin{equation}
  I(1)
  \;:=\;
  \int_0^1
  \frac{\sqrt{2}\,|w_j^\mathrm{out}|\, d\alpha}{\sqrt{2\bigl(\alpha^2 + (1-\alpha)^2\bigr)}}
  \;=\;
  \sqrt{2}\,|w_j^\mathrm{out}|\,\ln(1+\sqrt{2}),
  \label{eq:I1_def}
\end{equation}
where the closed form follows by the substitution $u = \alpha - \tfrac{1}{2}$ and the standard arsinh integral. The two permutation copies $(\alpha, 1-\alpha)$ and $(1-\alpha, \alpha)$ are already counted in the global $(n+1)!$ permutation factor, so we divide by $2! = 2$.

\paragraph{Summing over choices of split neuron.}
There are $n$ choices of which neuron $j$ to split, each contributing $I(1) = \sqrt{2}\,|w_j^\mathrm{out}|\,\ln(1+\sqrt{2})$ to the orbit integral. Summing over $j$ gives a total of $\sqrt{2}\,\ln(1+\sqrt{2}) \sum_{j=1}^n |w_j^\mathrm{out}| = \sqrt{2}\,\ln(1+\sqrt{2})\,\|w^\mathrm{out}\|_1$. Collecting all factors:
\begin{align}
  P_{n+1}
  \geq
  &\underbrace{(2b_{n+1})^{-2(n+1)}}_{\text{init density}}
  \;\cdot\;
  \underbrace{(n+1)!}_{\text{perm.\ copies}}
  \;\cdot\;
  \underbrace{(\epsilon/C)^{1/4}}_{\substack{\text{locally flat}\\\text{direction}}}
  \;\cdot\;
  \underbrace{\frac{(2\epsilon)^{n}\,\Vol(B_{2n})}{\sqrt{\det(H)}}}_{\text{ellipsoid vol.}}
  \;\cdot\;
  \underbrace{\frac{\sqrt{2}\,\ln(1+\sqrt{2})\,\|w^\mathrm{out}\|_1}{2}}_{\substack{\text{split neuron}\\\text{choices}}}.
  \label{eq:Pnk1}
\end{align}

\subsection{The ratio $P_{n+1}/P_n$}

Dividing~\eqref{eq:Pnk1} by~\eqref{eq:Pn}, and substituting Xavier initialization $b_m = 1/\sqrt{m}$:
\begin{align}
  \frac{P_{n+1}}{P_n}
  &\geq
  \underbrace{\frac{(2b_n)^{2n}}{(2b_{n+1})^{2(n+1)}}}_{\displaystyle
    = \frac{(n+1)^{n+1}}{4\, n^{n}}}
  \;\cdot\;
  \underbrace{\frac{(n+1)!}{n!}}_{\displaystyle = n+1}
  \;\cdot\;
  (\epsilon/C)^{1/4}
  \;\cdot\;
  \frac{\sqrt{2}\,\ln(1+\sqrt{2})\,\|w^\mathrm{out}\|_1}{2}
  \notag\\[6pt]
  &\;=\;
  \frac{\sqrt{2}\,\ln(1+\sqrt{2})}{8}
  \;\cdot\;
  \|w^\mathrm{out}\|_1\,(n+1)^2
  \left(\frac{n+1}{n}\right)^{n}
  (\epsilon/C)^{1/4}.
  \label{eq:ratio_k1}
\end{align}

\begin{remark}[Growth threshold]
All factors in~\eqref{eq:ratio_k1} are positive, so $P_{n+1}/P_n > 0$ always. The ratio exceeds $1$ if and only if
\begin{equation}
  \|w^\mathrm{out}\|_1\,(n+1)^2 \left(\frac{n+1}{n}\right)^{\!n}
  \;>\;
  \frac{8C^{1/4}}{\sqrt{2}\,\ln(1+\sqrt{2})\;\epsilon^{1/4}}.
  \label{eq:threshold_iff}
\end{equation}
Since $((n+1)/n)^n$ is increasing in $n$ with $((n+1)/n)^n \geq 2$ for all $n \geq 1$, a sufficient condition is
\begin{equation}
  \|w^\mathrm{out}\|_1\,(n+1)^2 >
  \frac{8C^{1/4}}{\sqrt{2}\,\ln(1+\sqrt{2})\;\epsilon^{1/4}}.
  \label{eq:threshold_sufficient}
\end{equation}
\end{remark}

\begin{remark}[Lower bound interpretation]
The formula~\eqref{eq:ratio_k1} counts only global minima arising from single-neuron splitting (the $k=1$ duplication symmetry). The true $P_{n+1}$ also receives contributions from other global minima not of this form, so~\eqref{eq:ratio_k1} provides a lower bound on the true ratio.
\end{remark}

\section{Related Work}
\label{sec:related}
%%%%%%%%%%%%%%%%%%%%%%%%%%%%%%%%%%%%%%%%%%%%%%%%%%%%%%%%%%%%%%%%%%%%%%%%%%%%%%%

\textbf{Geometry of overparameterization.} \citet{sussmann1992uniqueness} showed early on that for minimal-width networks, weights are determined up to permutation and sign-flip symmetries. \citet{FUKUMIZU2000317} later demonstrated that as width increases, the hierarchical structure of MLPs allows local minima to transform into saddle points, smoothing the landscape. \citet{simsek2021geometry} rigorously extended this to the overparameterized setting, identifying the continuous symmetries (neuron splitting and merging) that arise when width exceeds the functional minimum; their work provides the foundational groupoid structure for our analysis. Recently, \citet{10.5555/3692070.3693489} leveraged these symmetries for parameter recovery, showing that overparameterized student networks naturally "cluster" their weights around the underlying minimal representation. We build upon these results by showing that these symmetries act as diagonal preconditioning on the Hessian and scale the probability mass of global minima.

% \citet{kunin2021neural} take a complementary view through conservation laws: each continuous symmetry of the loss induces a conserved quantity along gradient flow, analogous to Noether's theorem, which constrains training dynamics beyond the static geometry of minima. \citet{grigsby2023hidden} catalog additional ``hidden'' symmetries in ReLU networks that arise from the piecewise-linear activation structure, including symmetries that exist only at specific weight configurations (e.g., when a neuron's pre-activation is zero on all training data).

\textbf{Loss landscape structure.} \citet{cooper2021global} proves that for networks with $d$ parameters and $n$ data points ($d > n$), the global minima generically form a $(d - n)$-dimensional submanifold rather than isolated points. \citet{nguyen2019connected} shows that all sublevel sets of the loss are connected under mild overparameterization, ruling out ``bad local valleys.'' Empirically, \citet{garipov2018loss} and \citet{draxler2018no} independently found that minima from different training runs can be connected by paths of near-constant loss. \citet{entezari2022role} conjectured that after accounting for permutation symmetry, SGD solutions from different initializations lie in a single basin; \citet{ainsworth2023git} gave practical algorithms for aligning models modulo these symmetries. Our work adds to this picture by showing that overparameterization symmetries not only connect minima but reshape the curvature at those minima. \citet{safran2015quality} bounded the probability of initializing in a basin with monotonic descent to a global minimum, and \citet{soudry2017exponentially} proved that sub-optimal local minima vanish exponentially in volume as width grows. Our \cref{sec:volume} provides a specific symmetry-based mechanism for these basin-quality improvements.

\textbf{Convergence theory.} \citet{jacot2018ntk} introduced the neural tangent kernel (NTK), showing that in the infinite-width limit, training reduces to kernel regression with a fixed kernel. Building on this, several works proved that gradient descent on sufficiently wide networks converges to global minima at a linear rate: \citet{du2019gradient} for deep residual networks, \citet{allenzhu2019convergence} for general deep networks, and \citet{zou2019sgd} for SGD on deep ReLU networks. These results typically require the width to be polynomial in the number of samples and the depth. \citet{chizat2019lazy} showed that this favorable convergence arises from a ``lazy training'' regime where parameters stay near initialization and the model behaves as its linearization, which can underperform feature learning in practice. Our approach is different: rather than analyzing dynamics in a linearized regime, we characterize the geometric properties (conditioning, volume) of minima at finite width.

\textbf{Implicit acceleration and deep linear networks.} Deep linear networks, where $f(x) = W_L \cdots W_1 x$, are a standard testbed for optimization theory. \citet{saxe2014exact} derived exact learning dynamics, showing that depth creates saddle-to-saddle transitions that progressively learn singular value components. \citet{arora2018optimization} showed that the resulting dynamics acts as implicit acceleration: the factored parameterization preconditions the optimization of the end-to-end matrix, with the effect depending on the singular values of the current solution. \citet{ghosh2025learning} extend this analysis beyond the edge of stability, characterizing oscillatory dynamics at large learning rates. We extend the preconditioning perspective in a different direction, from depth to width and from linear to general activations, showing that overparameterization symmetries provide diagonal preconditioning on the Hessian.

\textbf{Conditioning and curvature.} \citet{sagun2017hessian} showed empirically that overparameterized networks have Hessian spectra with a large bulk of near-zero eigenvalues and a few data-dependent outliers; the near-zero eigenvalues correspond to flat directions from overparameterization, while the outliers capture the curvature relevant to learning. \citet{ghorbani2019investigation} gave a more detailed eigenvalue density analysis, connecting spectral structure to properties of the data and architecture. \citet{guilleescuret2023nowrongturns} found that optimization paths in neural networks have surprisingly simple geometry despite the non-convexity of the loss. \citet{yang2022tensor} showed through $\mu$P that wider networks are more robust to learning rate choices, suggesting that width improves conditioning in a way that reduces hyperparameter sensitivity. Classical results on diagonal preconditioning \citep{forsythe1955best, collobert2004large} show that coordinate rescaling can improve the condition number when eigenvectors are not uniformly distributed across coordinates; this underpins our analysis in \cref{sec:preconditioning}.

\textbf{Normalization and symmetry exploitation.} Several practical methods explicitly navigate weight-space symmetries to find better-conditioned parameterizations. Weight normalization \citep{salimans2016weight} decouples weight magnitude from direction, moving along the rescaling orbit. Path-SGD \citep{neyshabur2015path} uses a path-regularized objective invariant to rescaling, so the trajectory is not distorted by the parameterization choice within an equivalence class. \citet{mishkin2025levelset} introduced level-set teleportation, which solves for a better-conditioned point on the same level set by optimizing over the symmetry group. All of these can be viewed as searching within a symmetry orbit for favorable parameterizations.

\newpage
\end{document}